\documentclass{article}

\usepackage[preprint]{neurips_2025}

\usepackage[utf8]{inputenc} 
\usepackage[T1]{fontenc}    
\usepackage{hyperref}       
\usepackage{url}            
\usepackage{booktabs}       
\usepackage{amsfonts}       
\usepackage{nicefrac}       
\usepackage{microtype}      
\usepackage{xcolor}         
\usepackage{float}          
\usepackage{longtable}      
\usepackage{graphicx}       

\usepackage{amsmath}
\usepackage{amssymb}
\usepackage{mathtools}
\usepackage{amsthm}

\usepackage[capitalize,noabbrev]{cleveref}

\theoremstyle{plain}
\newtheorem{theorem}{Theorem}[section]
\newtheorem{proposition}[theorem]{Proposition}

\theoremstyle{definition}

\theoremstyle{remark}

\usepackage[disable,textsize=tiny]{todonotes}

\usepackage{graphicx}
\usepackage{apptools}
\usepackage[flushleft]{threeparttable}
\usepackage{array,booktabs,makecell}
\usepackage{multirow}
\usepackage{nicefrac}
\usepackage{amsthm}
\usepackage{amsmath,amsfonts,bm}
\usepackage{wrapfig}
\usepackage{caption}
\usepackage{subcaption}
\usepackage{siunitx}
\usepackage{thm-restate}
\usepackage{nccmath}
\usepackage{empheq}
\usepackage{bbm}
\usepackage{tabularx}
\usepackage{placeins}
\usepackage{float}
\usepackage{algorithmic}
\usepackage{algorithm}
\usepackage{tabularx}

\usepackage{color}
\usepackage{colortbl}
\definecolor{bgcolor}{rgb}{0.76,0.88,0.50}
\definecolor{bgcolor0}{rgb}{0.93,0.99,1}
\definecolor{bgcolor1}{rgb}{0.8,1,1}
\definecolor{bgcolor2}{rgb}{0.8,1,0.8}
\definecolor{bgcolor3}{rgb}{0.50,0.90,0.50}
\usepackage{tcolorbox}
\usepackage{pifont}
\definecolor{mydarkgreen}{rgb}{0.15,0.5,0.26}
\definecolor{mydarkred}{rgb}{0.75,0.1,0.1}
\definecolor{highlight}{rgb}{0.9, 0.9, 0.9} 

\newcommand{\R}{\mathbb{R}} 

\theoremstyle{plain}

\makeatletter
\newcommand{\vast}{\bBigg@{4}}

\newcommand{\del}[1]{}

\theoremstyle{plain}

\theoremstyle{definition}

\usepackage{xspace}

\newcommand{\algname}[1]{{\color{gray}\small\sf#1}\xspace}

\title{Super-Tuning: From Activation-Aware Pruning to Sparse Fine-Tuning}

\author{%
  \begin{tabular}{c@{\hspace{1.4em}}c@{\hspace{1.4em}}c}
  Ivan Ilin &
  Philip Zmushko\footnotemark[2] &
  Peter Richt\'arik \\
  KAUST &
  KAUST &
  KAUST \\
  \texttt{\footnotesize ivan.ilin@kaust.edu.sa} &
  \texttt{\footnotesize zmushko.ph.a@gmail.com} &
  \texttt{\footnotesize peter.richtarik@kaust.edu.sa}
  \end{tabular}
}

\begin{document}

\maketitle
\begingroup
\renewcommand{\thefootnote}{\fnsymbol{footnote}}
\footnotetext[2]{Work done while Philip Zmushko was an intern at KAUST; he is now affiliated with ISTA and was affiliated with Yandex Research before coming to KAUST.}
\endgroup
\footnotetext[1]{Code repository: \url{https://github.com/vectozavr/SuperTuning}.}

\begin{abstract}
Large language models (LLMs) remain expensive to fine-tune because full-parameter updates require substantial memory, compute, and per-task storage. We study whether saliency signals originally developed for pruning can be reused to choose where a model should adapt. We propose \algname{Super}, a sparse parameter-efficient fine-tuning (PEFT) method that fixes a small trainable support using a \algname{Wanda}-style activation-weighted magnitude score~\citep{sun2023simple} computed from a calibration pass. We then introduce \algname{Supra}, a hybrid adapter that combines this sparse update with \algname{LoRA} while preserving a matched trainable-parameter budget through a simple budget-splitting rule. In single-seed Math17K arithmetic experiments on \texttt{Llama-3.2-1B} and \texttt{Meta-Llama-3-8B}, the best \algname{Super}/\algname{Supra} variants achieve the highest average accuracy among the tested schedule-selected adapter configurations. We also include a PaFi-style magnitude-only support as a closest training-free sparse baseline and find that low-score supports under both magnitude and \algname{Wanda}-style orderings can be effective. These results suggest that simple pruning-inspired orderings can provide useful fixed sparse supports for PEFT, especially when combined with low-rank adapters.
\end{abstract}

\section{Introduction}
\label{sec:intro}

Large Language Models (LLMs) have become central to modern natural language processing, achieving strong results in a wide array of tasks including question answering, summarization, code generation, and reasoning. However, their impressive capabilities come with a substantial cost: full fine-tuning of such models requires considerable computational resources, high memory consumption, and significant storage for each downstream task. This inefficiency becomes particularly problematic when deploying models in real-world settings with resource constraints or personalization requirements.

To address these limitations, \emph{parameter-efficient fine-tuning} (PEFT) methods have been proposed. These approaches update only a small fraction of the model's parameters, leaving the majority of weights frozen. Among the most successful PEFT techniques is \algname{LoRA} \citep{hu2021lora}, which injects low-rank trainable matrices into existing layers. \algname{LoRA} enables tuning with minimal parameter overhead and has been widely adopted for its balance of efficiency and performance.

Beyond low-rank adaptation, recent work has explored the use of \emph{sparse adaptation}, where only a small, selected subset of existing weights are updated. Methods such as \algname{SIFT} \citep{song2023sparse} and \algname{RoSA} \citep{nikdan2024rosa} select salient weights for fine-tuning based on importance scores or gradient signals. \algname{PaFi} \mbox{\citep{liao-etal-2023-parameter}} instead shows that a task-agnostic sparse mask can be generated without data or gradients by using only pretrained weight magnitudes. \algname{NeFT} \citep{xu2024let} takes this further by identifying and training only the most critical neurons. \algname{SpIEL} \citep{ansell2024scaling} introduces scalable sparse tuning with structured expert layers, and \algname{SAT} \citep{ma2024sparsity} proposes sparsity-accelerated training with carefully selected updates. \algname{$\text{S}^{2}\text{FT}$} \citep{yang2024s} combines sparsity with structured decomposition for efficient and generalizable tuning, while \algname{GIFT-SW} \citep{zhelnin2024gift} uses noise-injected fine-tuning on salient weights to improve robustness. Recent salience-aware sparse PEFT work further studies static and dynamic masks and compares several salience metrics for sparse fine-tuning \citep{liu2025refining}.

While sparse tuning methods are promising, they often require multiple training phases, gradient computations, or architectural modifications to identify important weights. A natural way to improve parameter-efficient tuning is to combine sparse and low-rank updates. This direction is also explored by \algname{SLTrain} \citep{han2024sltrain} in a pretraining setting and, more directly for PEFT, by \algname{RoSA} \citep{nikdan2024rosa}, which integrates sparse and low-rank adapters. Low-rank variants such as \algname{AdaLoRA} \citep{zhang2023adalora} and \algname{DoRA} \citep{liu2024dora} also show that the basic \algname{LoRA} parameterization can be improved by changing rank allocation or weight decomposition. We study whether a pruning-derived, activation-aware, training-free salience score can provide a useful fixed sparse support for PEFT, and whether this support can complement \algname{LoRA} under a transparent parameter-count budget.

In this paper, we introduce \algname{Super}, a sparse PEFT method based on \algname{Wanda}-ordered weights. Inspired by pruning techniques such as \algname{Wanda} \citep{sun2023simple}, \algname{Super} identifies and fine-tunes a small set of weights under an activation-weighted magnitude metric. Unlike sparse PEFT approaches that require additional gradient computations or training stages for weight selection, \algname{Super} relies on a simple, interpretable, training-free metric computed from calibration activations. We study both TopK and BottomK regimes, selecting respectively the largest and smallest \algname{Wanda}-style scores. In our schedule-selected arithmetic-reasoning comparisons, low-score supports often give the strongest sparse and sparse--low-rank variants, although the best direction can depend on model size, training schedule, and whether the sparse support is used alone or combined with \algname{LoRA}. We also test a PaFi-style magnitude-only selection rule, replacing the \algname{Wanda} score by \(|W_{ij}|\), to separate the effect of activation weighting from the simpler pretrained-magnitude prior and to provide a closest training-free sparse baseline.

We further propose \algname{Supra}, a hybrid method that combines \algname{Super} with \algname{LoRA} to jointly leverage \algname{Wanda}-selected sparse updates and low-rank adaptation. This design belongs to the same broad family as hybrid sparse--low-rank methods such as \algname{RoSA}~\citep{nikdan2024rosa}; the main distinction is that our sparse support is selected once using a training-free score, and we explicitly compare \algname{Wanda}-style and magnitude-only BottomK supports under a matched scalar-parameter budget. Our experiments fine-tune two Llama-family models on Math17K and evaluate arithmetic reasoning on AddSub, MultiArith, SingleEq, GSM8K, AQuA, and SVAMP under matched trainable-parameter budgets.

We summarize our key contributions as follows:
\begin{itemize}
    \item We propose \algname{Super}, a \algname{Wanda}-ordered sparse fine-tuning method that updates only a small set of weights selected without training or gradient statistics.
    \item We introduce \algname{Supra}, a hybrid PEFT strategy that combines \algname{Wanda}-selected sparse updates with \algname{LoRA} under an explicit parameter-count budget.
    \item We use a simple budget-matched rank conversion rule that maps a fixed parameter-count budget into a layerwise \algname{LoRA} rank; unlike adaptive rank-allocation methods, this rule is a transparent budget conversion rather than a learned or importance-weighted allocation.
    \item We evaluate \algname{LoRA}, \algname{SIFT}, \algname{RoSA}, random sparse supports, PaFi-style magnitude supports, \algname{Super}, and \algname{Supra} on Math17K arithmetic fine-tuning under matched trainable-parameter budgets. In the reported single-seed schedule-selected comparisons, the best sparse and sparse--low-rank variants achieve the highest average accuracy among the tested adapter configurations.
\end{itemize}

The implementation is available at \url{https://github.com/vectozavr/SuperTuning}.

\Cref{tab:main_text_llama1b_compact,tab:main_text_llama8b_compact} summarize the primary matched-budget arithmetic comparisons, and \Cref{tab:main_text_method_summary} collects the corresponding sparse-mask rules, average accuracies, and average perplexities. \algname{Supra-Mag} denotes the same sparse--low-rank budget split as \algname{Supra}, but with the sparse support selected by bottom \(|W_{ij}|\) rather than by the \algname{Wanda} score. The complete per-benchmark accuracy and perplexity tables for the one- and three-epoch Math17K runs are reported in \Cref{sec:llama1b_math17k_1epoch_appendix,sec:llama1b_math17k_3epoch_appendix,sec:llama8b_math17k_1epoch_appendix,sec:llama8b_math17k_appendix}.

The compact tables below are schedule-selected summaries of the full one- and three-epoch grids reported in the appendix. They are intended to highlight the best observed configuration for each method family under the matched-budget protocol. The same reading convention applies to both  tables: for each method, we report the one- or three-epoch run with the higher average exact-answer accuracy, and the \textbf{Epochs} column gives the selected schedule. Bold marks the best observed matched-budget adapter average, shaded rows are proposed methods, all entries are exact-answer accuracy (\%) across AddSub, MultiArith, SingleEq, GSM8K, AQuA, and SVAMP.

\subsection{Notation} All key notation used in this paper is summarized in a tabular form in \Cref{sec:notation}; see Table~\ref{tab:notation_table}.

\begin{table}[!t]
\centering
\scriptsize
\caption{Schedule-selected Math17K arithmetic comparison for \texttt{Llama-3.2-1B} under rank-equivalent budget \(r_0=8\).}
\label{tab:main_text_llama1b_compact}
\resizebox{\textwidth}{!}{%
\begin{tabular}{lccccccccc}
\toprule
\textbf{Method} & \textbf{Epochs} & \textbf{Selected LR} & \textbf{AddSub} & \textbf{MultiArith} & \textbf{SingleEq} & \textbf{GSM8K} & \textbf{AQuA} & \textbf{SVAMP} & \textbf{Average} \\
\midrule
\textsc{Base} (frozen) & -- & -- & 13.67 & 4.67 & 21.46 & 2.81 & 21.65 & 11.60 & 12.64 \\
\midrule
Full fine-tuning & 3 & $5\cdot10^{-5}$ & 80.00 & 91.67 & 85.63 & 44.05 & 28.74 & 64.10 & 65.70 \\
\midrule
\algname{LoRA} & 3 & $5\cdot10^{-4}$ & 70.63 & 90.83 & 85.24 & 35.33 & 26.77 & 57.60 & 61.07 \\
\algname{RoSA} & 1 & $5\cdot10^{-4}$ & 71.65 & 89.17 & 78.35 & 27.98 & 24.80 & 53.70 & 57.61 \\
\algname{SIFT} (TopK) & 1 & $10^{-4}$ & 60.51 & 83.17 & 63.19 & 23.43 & 3.94 & 40.30 & 45.75 \\
\algname{SIFT} (RandK) & 1 & $10^{-3}$ & 62.28 & 90.83 & 73.62 & 31.46 & 17.72 & 48.70 & 54.10 \\
\algname{Magnitude} (BottomK) & 1 & $10^{-3}$ & 73.92 & 90.50 & 81.69 & 30.25 & 24.80 & 54.70 & 59.31 \\
\midrule
\rowcolor{gray!12}
\algname{Super} (BottomK) & 3 & $5\cdot10^{-4}$ & 69.11 & 85.50 & 70.47 & 34.50 & 25.59 & 47.60 & 55.46 \\
\rowcolor{gray!12}
\algname{Supra-Mag} (BottomK, $\lambda=0.3$) & 1 & $5\cdot10^{-4}$ & 76.71 & 91.67 & 81.10 & 22.97 & 27.56 & 55.10 & 59.18 \\
\rowcolor{gray!12}
\algname{Supra} (BottomK, $\lambda=0.8$) & 3 & $5\cdot10^{-4}$ & 82.53 & 85.50 & 84.45 & 36.16 & 24.41 & 60.30 & \textbf{62.23} \\
\bottomrule
\end{tabular}
}
\end{table}

\begin{table}[!t]
\centering
\scriptsize
\caption{Schedule-selected Math17K arithmetic comparison for \texttt{Meta-Llama-3-8B} under rank-equivalent budget \(r_0=8\).}
\label{tab:main_text_llama8b_compact}
\resizebox{\textwidth}{!}{%
\begin{tabular}{lccccccccc}
\toprule
\textbf{Method} & \textbf{Epochs} & \textbf{Selected LR} & \textbf{AddSub} & \textbf{MultiArith} & \textbf{SingleEq} & \textbf{GSM8K} & \textbf{AQuA} & \textbf{SVAMP} & \textbf{Average} \\
\midrule
\textsc{Base} (frozen) & -- & -- & 22.53 & 21.00 & 36.42 & 10.24 & 24.02 & 24.10 & 23.05 \\
\midrule
Full fine-tuning & 3 & $5\cdot10^{-5}$ & 90.63 & 96.17 & 93.70 & 54.36 & 38.19 & 76.20 & 74.87 \\
\midrule
\algname{LoRA} & 1 & $5\cdot10^{-4}$ & 86.33 & 97.67 & 93.70 & 66.34 & 24.41 & 70.60 & 73.17 \\
\algname{RoSA} & 1 & $10^{-4}$ & 86.84 & 97.50 & 92.72 & 58.45 & 35.83 & 79.60 & 75.16 \\
\algname{SIFT} (TopK) & 1 & $10^{-4}$ & 86.58 & 98.50 & 94.69 & 63.76 & 35.04 & 79.10 & 76.28 \\
\algname{SIFT} (RandK) & 1 & $10^{-3}$ & 84.30 & 98.33 & 93.50 & 63.99 & 45.28 & 81.60 & 77.83 \\
\algname{Magnitude} (BottomK) & 1 & $10^{-3}$ & 90.89 & 97.33 & 95.08 & 68.01 & 41.73 & 81.10 & 79.02 \\
\midrule
\rowcolor{gray!12}
\algname{Super} (BottomK) & 3 & $10^{-4}$ & 72.15 & 95.67 & 88.58 & 66.34 & 42.91 & 70.70 & 72.73 \\
\rowcolor{gray!12}
\algname{Supra} (BottomK, $\lambda=0.3$) & 1 & $5\cdot10^{-4}$ & 89.11 & 95.33 & 96.06 & 67.63 & 41.34 & 82.50 & 78.66 \\
\rowcolor{gray!12}
\algname{Supra-Mag} (BottomK, $\lambda=0.3$) & 1 & $5\cdot10^{-4}$ & 92.91 & 98.50 & 96.85 & 69.83 & 35.43 & 81.20 & \textbf{79.12} \\
\bottomrule
\end{tabular}
}
\end{table}

\begin{table}[!t]
\centering
\scriptsize
\caption{Summary of the schedule-selected main-text rows. Fine-tuning dataset: Math17K; rank-equivalent budget: \(r_0=8\). The score column gives the per-weight scalar used to rank candidate sparse coordinates; \(X_{j:}\) denotes the calibration activation vector feeding column \(j\) of a layer. Average accuracy is exact-answer accuracy (\%) across AddSub, MultiArith, SingleEq, GSM8K, AQuA, and SVAMP; average PPL is test perplexity over the same benchmark set, where lower is better.}
\label{tab:main_text_method_summary}
\resizebox{\textwidth}{!}{%
\begin{tabular}{llcccc}
\toprule
\textbf{Method} & \textbf{Per-weight score} & \multicolumn{2}{c}{\textbf{\texttt{Llama-3.2-1B}}} & \multicolumn{2}{c}{\textbf{\texttt{Meta-Llama-3-8B}}} \\
\cmidrule(lr){3-4}\cmidrule(lr){5-6}
 & & \textbf{Acc.} & \textbf{PPL} & \textbf{Acc.} & \textbf{PPL} \\
\midrule
\textsc{Base} (frozen) & -- & 12.64 & 2.73 & 23.05 & 2.22 \\
\midrule
Full fine-tuning & all parameters & 65.70 & 1.11 & 74.87 & 1.07 \\
\midrule
\algname{LoRA} & -- & 61.07 & 1.17 & 73.17 & 1.14 \\
\algname{RoSA} & $\left|\nicefrac{\partial\mathcal L}{\partial W_{ij}}\right|$ & 57.61 & 1.27 & 75.16 & 1.18 \\
\algname{SIFT} (TopK) & $\left|\nicefrac{\partial\mathcal L}{\partial W_{ij}}\right|$ & 45.75 & 1.33 & 76.28 & 1.18 \\
\algname{SIFT} (RandK) & $\mathrm{Unif}(0,1)$ & 54.10 & 1.25 & 77.83 & 1.13 \\
\algname{Magnitude} (BottomK) & $\left|W_{ij}\right|$ & 59.31 & 1.24 & 79.02 & 1.13 \\
\midrule
\rowcolor{gray!12}
\algname{Super} (BottomK) & $\left|W_{ij}\right|\,\left\|X_{j:}\right\|_2$ & 55.46 & 1.21 & 72.73 & 1.31 \\
\rowcolor{gray!12}
\algname{Supra} (BottomK, $\lambda=0.3$) & $\left|W_{ij}\right|\,\left\|X_{j:}\right\|_2$ & 55.03 & 1.26 & 78.66 & 1.14 \\
\rowcolor{gray!12}
\algname{Supra} (BottomK, $\lambda=0.8$) & $\left|W_{ij}\right|\,\left\|X_{j:}\right\|_2$ & \textbf{62.23} & 1.16 & 78.28 & 1.14 \\
\rowcolor{gray!12}
\algname{Supra-Mag} (BottomK, $\lambda=0.3$) & $\left|W_{ij}\right|$ & 59.18 & 1.25 & \textbf{79.12} & 1.14 \\
\bottomrule
\end{tabular}
}
\end{table}

\section{Related Work}
\label{sec:literature}

\subsection{Parameter–Efficient Fine‑Tuning (PEFT)}

The prohibitive compute and memory cost of updating all parameters of a large language model (LLM) has motivated a rich literature on \emph{parameter–efficient fine‑tuning} (PEFT). Classical adapter-style methods insert small \emph{dense} modules into each Transformer block (e.g., adapters, prefix-tuning, and \algname{LoRA}), greatly reducing trainable parameters, optimizer state, and per-task storage. Complementary work studies \emph{sparse} or otherwise \emph{structured} updates that explicitly select a subset of parameters to train, concentrating adaptation in a smaller set of task-specific degrees of freedom.

\subsection{Sparse fine‑tuning}

Several papers report that high accuracy can be recovered even on reasoning‑heavy tasks -- when a model is updated on only a small fraction of its weights.  
\citet{song2023sparse} analyze the PAC‑Bayesian generalization bound of PEFT and propose \algname{SIFT}, a gradient‑based algorithm that activates at most $k$ parameters per layer during training.  
\citet{ansell2024scaling} show that dynamically grown–pruned sparse deltas scale to 13‑B‑parameter LLaMA‑2 while retaining memory proportional to the sparsity pattern rather than to model size.  
Orthogonally, \citet{ma2024sparsity} exploit the observation that only a minority of neurons fire on any given example and skip the forward/backward pass of inactive neurons to accelerate both continual pre‑training and supervised fine‑tuning by up to $45\%$.  
Building on structured sparsity, \citet{yang2024s} introduce \algname{S$^{2}$FT}, which selects a small set of attention heads and MLP channels, then co‑permutes weight matrices so the selected components form dense sub‑matrices that can be trained efficiently with ordinary GEMM kernels.  

For unstructured weight-level sparse masks, different methods mainly differ in the scalar score used to rank candidate coordinates. Given a fine-tuning loss \(\mathcal L\), we write the gradient-magnitude score for an entry \(W_{ij}\) as
\[
    G_{ij}=\left|\frac{\partial \mathcal L}{\partial W_{ij}}\right|.
\]
This is the score notation used in our method-summary table for gradient-based sparse selection methods such as \algname{SIFT} and the sparse component of \algname{RoSA}. In contrast, \algname{PaFi} uses a data-free pretrained-magnitude score,
\[
    B_{ij}=|W_{ij}|,
\]
selecting sparse fine-tuning coordinates from the magnitude ordering of the original weights \citep{liao-etal-2023-parameter}. Our \algname{Magnitude} and \algname{Supra-Mag} rows use this PaFi-style score, while \algname{Super} and the default \algname{Supra} rows use the activation-weighted \algname{Wanda} score introduced in \Cref{sec:adaptation}.

\algname{PaFi} is especially relevant because it shows that task-agnostic sparse masks can be built from pretrained weight magnitudes alone. We therefore treat magnitude-only BottomK selection as a PaFi-style sparse baseline. Recent work such as \algname{GaLLoP} also highlights the usefulness of low-magnitude parameters for sparse fine-tuning, combining small pretrained magnitudes with large downstream gradient magnitudes~\citep{choudhary2025gallop}. In contrast, our default \algname{Super} masks are selected without downstream gradients, using a training-free \algname{Wanda}-style activation-weighted score computed from a calibration pass. We also include a magnitude-only support as a PaFi-style baseline and study its combination with \algname{LoRA}.

\subsection{Hybrid low‑rank \& sparse methods}

A complementary line of work seeks to combine the representational benefits of low‑rank adaptation with the compactness of sparsity. \algname{RoSA} is closely related because it combines sparse and low-rank adaptation for PEFT: \citet{nikdan2024rosa} decompose the update into a low‑rank adapter plus a very sparse residual and optimize both jointly. Our novelty is not that sparse and low-rank updates can be combined, but that the sparse support is fixed by a simple training-free ordering and that we systematically compare \algname{Wanda}-style activation-weighted supports, magnitude-only supports, TopK/BottomK directions, and their sparse--low-rank hybrids under matched scalar-parameter budgets. For the pre‑training regime, \citet{han2024sltrain} factorize each linear layer into a low‑rank term and a \emph{fixed‑support} sparse term, achieving up to $73\%$ memory savings while matching full‑rank performance.

\subsection{Fine‑tuning at the neuron or weight level}

Moving to finer granularity, \citet{xu2024let} supervise the updates at the level of individual neurons (\algname{NeFT}), explicitly identifying task‑relevant neurons whose small subset updates can achieve higher accuracy than full‑parameter fine‑tuning.  \citet{zhelnin2024gift} extend the idea of salient‑weight selection with \algname{GIFT‑SW}, which injects Gaussian noise into \emph{non‑salient} columns while learning only the salient ones, closing the gap to full fine‑tuning under the same compute budget.

\subsection{Connections to pruning and outlier‑aware updates}

Pruning and sparse fine-tuning can be viewed as complementary mask-selection problems. Pruning asks which parameters can be removed or frozen while preserving the pretrained function; sparse fine-tuning asks where a limited adaptation budget should be allocated. This connection has appeared in several forms: \algname{FISH Mask} uses Fisher information to precompute fixed sparse update masks \citep{sung2021training}, \algname{PaFi} selects low-magnitude pretrained parameters for sparse fine-tuning and explicitly frames the procedure as pruning-and-fine-tuning \citep{liao-etal-2023-parameter}, row-based sparse fine-tuning uses neural-pruning feature-importance metrics to select structured update supports \citep{li2025efficient}, and \algname{RoCoFT} studies row/column updates with \algname{Wanda}-style row and column scoring among its selection strategies \citep{kowsher-etal-2025-rocoft}. In this work, we turn this pruning-to-fine-tuning connection into a simple weight-level PEFT method: a calibration pass computes activation-aware \algname{Wanda}-style scores, and the resulting ordering defines a fixed sparse set of original model weights to train under a matched parameter budget. We additionally include PaFi-style magnitude masks, which remove the activation term and rank coordinates only by \(|W_{ij}|\), as a direct comparison to the activation-aware \algname{Wanda} ordering.

Our method is inspired by pruning research that measures weight saliency $S_{kq}$ for every weight $W_{kq}$ from a linear layer $W \in \R^{c \times b}$ before deleting parameters. Using \(\|\cdot\|_2\) for the Euclidean norm, the \algname{Wanda} metric multiplies a weight’s magnitude by the norm of its input activation:
\begin{equation}
    \label{eq:wanda_metric}
    S_{kq} = \left(|W_{kq}|\|X_{q:}\|_2\right)^2,
\end{equation}
where $X_{q:} \in \R^{1 \times a}$ is the $q^{\text{th}}$ row of the input matrix $X \in \R^{b \times a}$, and \(a\) is the number of calibration token positions.
The squared form in \Cref{eq:wanda_metric} has an exact single-weight reconstruction interpretation: as shown later in \Cref{prop:wanda_single_weight}, zeroing \(W_{kq}\) changes the layer output by squared Frobenius norm \(W_{kq}^2\|X_{q:}\|_2^2\).

The metric $S_{kq}$ is used to approximate output-level importance and enables one-shot pruning without weight updates \citep{sun2023simple}. In pruning, low-\algname{Wanda}-score weights are natural candidates for removal because they have small estimated contribution to the layer-output reconstruction proxy. In sparse fine-tuning, we instead keep the model intact and ask whether this same activation-aware ordering can identify a small set of weights that can be safely and effectively trained. We therefore study both TopK and BottomK masks, selecting respectively the largest and smallest \algname{Wanda}-style scores under the same sparse budget. A natural intermediate possibility is to split the sparse budget between both ends of the score ordering; we evaluate this mixed TopK--BottomK support in \Cref{sec:hybrid_super_appendix}.

Our experiments suggest that BottomK entries often form useful sparse fine-tuning supports in this setting. This strengthens the connection to pruning: the weights that appear least critical for preserving the pretrained layer output can also provide useful degrees of freedom for task adaptation. One possible explanation is that high-score weights help preserve the pretrained computation, so perturbing them can be risky, while low-score weights can be changed more freely. Another is optimization-related: because BottomK entries are less sensitive under the reconstruction proxy, they may tolerate larger learning rates and adapt faster under a limited training budget. Thus, the \algname{Wanda}-style score should be viewed as an activation-aware ordering signal for sparse fine-tuning, not only as a pruning metric.

\section{Adaptation of Large Language Models}
\label{sec:adaptation}

\subsection{Notation}
Let \( \mathcal{N} \) denote a pre-trained Large Language Model (LLM), and let \( \mathcal{W} = \{W^{(1)}, W^{(2)}, \ldots, W^{(m)}\} \) be the collection of fully connected weight matrices in \( \mathcal{N} \) considered for adaptation, including those within attention sublayers, with each \( W^{(i)} \in \R^{c_i \times b_i} \) for \( 1 \leq i \leq m \). We use \(\mathcal A\subseteq\{1,\ldots,m\}\) for the index set of matrices actually adapted by a given method; when all candidate matrices are adapted, \(\mathcal A=\{1,\ldots,m\}\). Let the vector \( \bar{w} \in \R^{\bar{d}} \) represent the remaining parameters of \( \mathcal{N} \), such as biases and normalization parameters, concatenated into a single vector. Given a dataset \( \mathcal{D} \) and a loss function \( \mathcal{L}(\mathcal{D}; \mathcal{W}, \bar{w}) \), the full fine-tuning (FFT) of \( \mathcal{N} \) can be formulated as the following optimization problem:
\begin{equation}
    \min_{\mathcal{W}, \bar{w}} \mathcal{L}(\mathcal{D}; \mathcal{W}, \bar{w}).
    \label{eq:fft}
\end{equation}
Due to the large scale of modern LLMs -- often comprising billions of parameters -- performing FFT is computationally expensive and memory-intensive, making it impractical on standard GPUs. A practical alternative is to apply lightweight modifications known as \emph{adapters}, which we formalize next.

Let \( \Delta = \{\Delta^{(1)}, \Delta^{(2)}, \ldots, \Delta^{(m)}\} \) denote perturbations applied to the fully connected weights, with \( \Delta^{(i)} \in \R^{c_i \times b_i} \) for all \( i \). Define the adapted weights as \( \mathcal{W} + \Delta = \{W^{(1)} + \Delta^{(1)}, W^{(2)} + \Delta^{(2)}, \ldots, W^{(m)} + \Delta^{(m)}\} \), and let \( \bar{\delta} \in \R^{\bar{d}} \) represent perturbations to \( \bar{w} \). The adapted model is then obtained by solving:
\begin{equation}
    \min_{\Delta, \bar{\delta}} \mathcal{L}(\mathcal{D}; \mathcal{W} + \Delta, \bar{w} + \bar{\delta}) \quad \text{s.t.} \quad \mathcal{C}(\Delta, \bar{\delta}),
    \label{eq:adapter}
\end{equation}
where \( \mathcal{C}(\Delta, \bar{\delta}) \) encodes constraints on the perturbations (e.g., low-rank or sparse structure) to reduce the number of optimized and stored task-specific parameters. Such constraints may also reduce compute when paired with suitable structure or kernels, but unstructured sparse or low-rank adapters still require forward and backward propagation through the frozen pretrained model. Notably, if no constraints are imposed, this setting reduces to standard FFT.

In this work, we focus on the common scenario where \( \bar{\delta} = 0 \), though in principle \( \bar{w} \) can also be fine-tuned, especially given its relatively small size compared to \( \mathcal{W} \). Our reported experiments adapt the full set of standard Llama linear projections listed in \Cref{sec:experiments}, while the notation also allows partial adaptation through the index set \(\mathcal A\).

\paragraph{LoRA: Low-Rank Adaptation.} Low-rank adaptation (\algname{LoRA})~\citep{hu2021lora} constrains each perturbation \( \Delta^{(i)} \) to be low-rank, specifically:
\begin{equation}
    \min_{\Delta} \mathcal{L}(\mathcal{D}; \mathcal{W} + \Delta, \bar{w}) \quad \text{s.t.} \quad \text{rank}(\Delta^{(i)}) \leq r \quad \forall\, i.
    \label{eq:lora}
\end{equation}
In the usual implementation, the update is parameterized as
\begin{equation}
    \Delta^{(i)} = \gamma_i L^{(i)} R^{(i)},
    \qquad
    L^{(i)} \in \R^{c_i \times r_i}, \quad
    R^{(i)} \in \R^{r_i \times b_i},
    \label{eq:lora_parameterization}
\end{equation}
where \(\gamma_i\) is the \algname{LoRA} scaling factor. This uses \(r_i(c_i+b_i)\) trainable scalar parameters for layer \(i\). We use this implementation parameter count for budget matching because these are the scalar adapter weights optimized and stored by standard \algname{LoRA}.

\paragraph{SpA: Sparse Adaptation.} Using \(\|\cdot\|_0\) to denote the number of nonzero entries, sparse adaptation (SpA)~\cite{sung2021training} enforces sparsity in the perturbations:
\begin{equation}
    \min_{\Delta} \mathcal{L}(\mathcal{D}; \mathcal{W} + \Delta, \bar{w}) \quad \text{s.t.} \quad \|\Delta^{(i)}\|_0 \leq p c_i b_i \quad \forall\, i,
    \label{eq:spa}
\end{equation}
where \( p \in (0, 1] \) is the sparsity density. It is common to fix the sparse support during training, thereby reducing the number of trainable parameters by a factor of \( p \).

\subsection{Super: Mask Generation Algorithm for SpA.}

Let \(\mathcal{D}_{\mathrm{cal}}\) denote a tokenized calibration set. For a layer with input dimension \(b\), running the frozen model on \(\mathcal{D}_{\mathrm{cal}}\) yields the layer input activation matrix \(X\in\R^{b\times a}\), where \(a\) is the total number of calibration token positions. For fixed calibration activations \(X\), we introduce the mask-generation mapping
\[
    \psi_X:\R^{c\times b}\times\{0,\ldots,cb\}\times\{\mathrm{TopK},\mathrm{BottomK}\}
    \to \{0,1\}^{c\times b}.
\]
It takes a weight matrix \(W\in\R^{c\times b}\), a target number of sparse entries \(s\in\{0,\ldots,cb\}\), and a selection direction \(\rho\in\{\mathrm{TopK},\mathrm{BottomK}\}\), then generates the corresponding fine-tuning mask \(\psi_X(W,s,\rho)\). For each weight \(W_{ij}\), it evaluates the score
\[
    A_{ij}=|W_{ij}|\|X_{j:}\|_2.
\]
We use this activation-weighted magnitude score as a selection rule for fine-tuning: the scores are flattened within each layer, and the mask contains either the global top \(s\) entries or the global bottom \(s\) entries. Because the square is monotone on nonnegative values, this ordering is the same as the ordering induced by the squared \algname{Wanda} score in \Cref{eq:wanda_metric}.

Our use of the \algname{Wanda} score differs from the original pruning rule in selection granularity: we flatten scores within each adapted layer and select a global layerwise support, whereas pruning implementations often impose row-wise budgets. We use the term \algname{Wanda}-style to emphasize that we reuse the activation-weighted magnitude score, not the full pruning procedure. Row-wise support selection is a natural follow-up but is outside the scope of the present experiments.

The output is a binary mask with ones at these selected positions:
\begin{equation}
\label{eq:psi_def}
\psi_X(W,s,\rho)_{ij}
\coloneqq
\begin{cases}
1, & (i,j)\text{ is among the }s\text{ largest entries of }A,\quad \rho=\mathrm{TopK},\\
1, & (i,j)\text{ is among the }s\text{ smallest entries of }A,\quad \rho=\mathrm{BottomK},\\
0, & \text{otherwise.}
\end{cases}
\end{equation}
Refer to \Cref{sec:wanda_appendix} for additional examples.

Alongside the \algname{Wanda}-style mask, we also evaluate a PaFi-style magnitude-only rule \citep{liao-etal-2023-parameter}. This variant uses
\[
    B_{ij}=|W_{ij}|
\]
in place of \(A_{ij}\) and applies the same TopK or BottomK selection direction and the same sparse budget. Thus the magnitude-only sparse baseline and \algname{Supra-Mag} differ from \algname{Super} and the default \algname{Supra} only in the ranking score used to construct the fixed sparse support: \(|W_{ij}|\) rather than \(|W_{ij}|\|X_{j:}\|_2\). This lets us test whether the strong sparse supports come from the activation-weighted \algname{Wanda} proxy itself or from the simpler pretrained-magnitude prior used by PaFi.

One can also combine both ends of the score ordering by assigning part of the sparse budget to TopK entries and the remainder to BottomK entries. We report this mixed TopK--BottomK ablation in \Cref{sec:hybrid_super_appendix}; it preserves the same sparse parameter budget, but in our experiments it does not significantly improve over the pure BottomK support, which remains a strong choice in the schedule-selected comparisons.

The resulting mask-generation procedure is summarized in \Cref{alg:super_mask_generation}.

\begin{algorithm}[ht]    
    \caption{SpA mask generation by \algname{Wanda}-style scores}
    \label{alg:super_mask_generation}
    \begin{algorithmic}[1]
    \STATE \textbf{Input:} Adapted-layer index set \(\mathcal A\subseteq\{1,\ldots,m\}\) in a frozen pretrained model, weights \(\{W^{(i)}\in\R^{c_i\times b_i}\}_{i\in\mathcal A}\), tokenized calibration set \(\mathcal{D}_{\mathrm{cal}}\), sparse budgets \(s_i\in\{0,\ldots,c_i b_i\}\), selection direction \(\rho\in\{\mathrm{TopK},\mathrm{BottomK}\}\).
    \STATE Run the frozen model on \(\mathcal{D}_{\mathrm{cal}}\) and collect input activations \(X^{(i)}\in\R^{b_i\times a}\) for each \(i\in\mathcal A\).
    \FOR{each adapted layer \(i\in\mathcal A\)}
        \STATE Compute \(M^{(i)}=\psi_{X^{(i)}}(W^{(i)},s_i,\rho)\) using \Cref{eq:psi_def}.
    \ENDFOR
    \STATE \textbf{Return:} Fixed sparse masks \(\{M^{(i)}\in\{0,1\}^{c_i\times b_i}\}_{i\in\mathcal A}\).
    \end{algorithmic}
\end{algorithm}

Unlike many existing methods, our approach does not require a complex or computationally intensive procedure for selecting the trainable parameters of the sparse adapter. Instead, we adopt a simple strategy based on one calibration pass through the model. We refer to this method as \textbf{Super}--short for \emph{Selective Update of Parameters via Extreme Ranking}.

For layer \(i\), let \(W^{(i)}\) be the frozen pretrained matrix and let \(X^{(i)}\) be the calibration activations collected at that layer. Given a sparse budget \(s_i\) and selection direction \(\rho\), \algname{Super} constructs
\begin{equation}
    M^{(i)} = \psi_{X^{(i)}}(W^{(i)}, s_i,\rho)
    \label{eq:super_mask}
\end{equation}
and trains a sparse delta on this fixed support. We write \(A \odot B\) for the element-wise (Hadamard) product of matrices with the same shape:
\begin{equation}
    \widehat W^{(i)} = W^{(i)} + M^{(i)} \odot U^{(i)},
    \label{eq:super_update}
\end{equation}
where \(U^{(i)} \in \R^{c_i \times b_i}\) is trainable only on the entries selected by \(M^{(i)}\). Equivalently, one may view the selected entries of \(W^{(i)}\) as being updated in place while all unselected entries remain frozen.

The sparse delta is initialized at zero, so the initial adapted layer equals the pretrained layer. For adapted-layer index set \(\mathcal A\), the \algname{Super} objective can be written as
\begin{equation}
    \min_{\{U^{(i)}\}_{i\in\mathcal A}}
    \mathcal{L}\!\left(\mathcal{D}; \{\widehat W^{(i)}\}_{i\in\mathcal A}, \{W^{(j)}\}_{j\notin\mathcal A}, \bar w\right),
    \label{eq:super_objective}
\end{equation}
with all pretrained weights outside the selected sparse entries frozen.

The core idea behind \algname{Super} is to leverage the \algname{Wanda} metric, which estimates the single-weight contribution to layer-wise \(\ell_2\) reconstruction error upon removal. The following identity gives the exact single-weight form used by the score.

\begin{proposition}[Single-weight reconstruction identity]
\label{prop:wanda_single_weight}
Let \(W \in \R^{c \times b}\), \(X \in \R^{b \times a}\), \(Y=WX\), and let \(\|\cdot\|_F\) denote the Frobenius norm. Let \(W^{(-kq)}\) equal \(W\) except that entry \(W_{kq}\) is set to zero. Then
\begin{equation}
    \|(W-W^{(-kq)})X\|_F^2
    =
    W_{kq}^2 \|X_{q:}\|_2^2.
    \label{eq:wanda_single_weight_identity}
\end{equation}
\end{proposition}

\begin{proof}
The difference \(W-W^{(-kq)}\) has a single nonzero entry, so \((W-W^{(-kq)})X = W_{kq} e_k X_{q:}\), where \(e_k\) is the \(k\)-th standard basis vector. Taking the Frobenius norm gives \(W_{kq}^2\|X_{q:}\|_2^2\).
\end{proof}

Importantly, this identity concerns the reconstruction effect of deleting a pretrained weight, whereas \algname{Super} uses the same score only as a support-selection heuristic for additive sparse updates. Once a coordinate is selected, the learned delta is not constrained to equal the negative pretrained weight. Therefore, the \algname{Wanda}-style score should not be interpreted as an optimality criterion for sparse fine-tuning. Rather, it provides a training-free ordering of coordinates by their estimated sensitivity in the frozen pretrained model.

For a set of selected weights, the exact reconstruction effect also contains cross-terms between activation rows. Writing \(\langle \cdot,\cdot\rangle\) for the Euclidean inner product, for any binary mask \(M\in\{0,1\}^{c\times b}\),
\begin{equation}
    \|(M\odot W)X\|_F^2
    =
    \sum_{k=1}^{c}\sum_{q=1}^{b}\sum_{r=1}^{b}
    M_{kq}M_{kr}W_{kq}W_{kr}\langle X_{q:},X_{r:}\rangle.
    \label{eq:wanda_cross_terms}
\end{equation}
The independent bottom-\(s_i\) selection used by \algname{Super} corresponds to the diagonal approximation
\begin{equation}
    \sum_{k=1}^{c}\sum_{q=1}^{b}
    M_{kq}W_{kq}^2\|X_{q:}\|_2^2,
    \label{eq:wanda_diagonal_approx}
\end{equation}
which drops all \(q\neq r\) cross-terms. Thus, \algname{Wanda}-style selection should be understood as a diagonal reconstruction heuristic rather than an exact global optimization criterion.

Our central hypothesis is that training can be concentrated on parameters ordered by this reconstruction proxy to obtain useful adaptation under small parameter-count budgets. By selecting a fixed support according to the \algname{Wanda}-style metric, \algname{Super} provides a simple method for sparse fine-tuning that aligns with this reconstruction-based heuristic.

\subsection{\texorpdfstring{\algname{Supra}}{Supra}: Budget-Matched Low-Rank and Sparse Adaptation.}
\label{sec:supra}

\begin{figure*}[t!]
    \centering
    \includegraphics[width=0.8\textwidth]{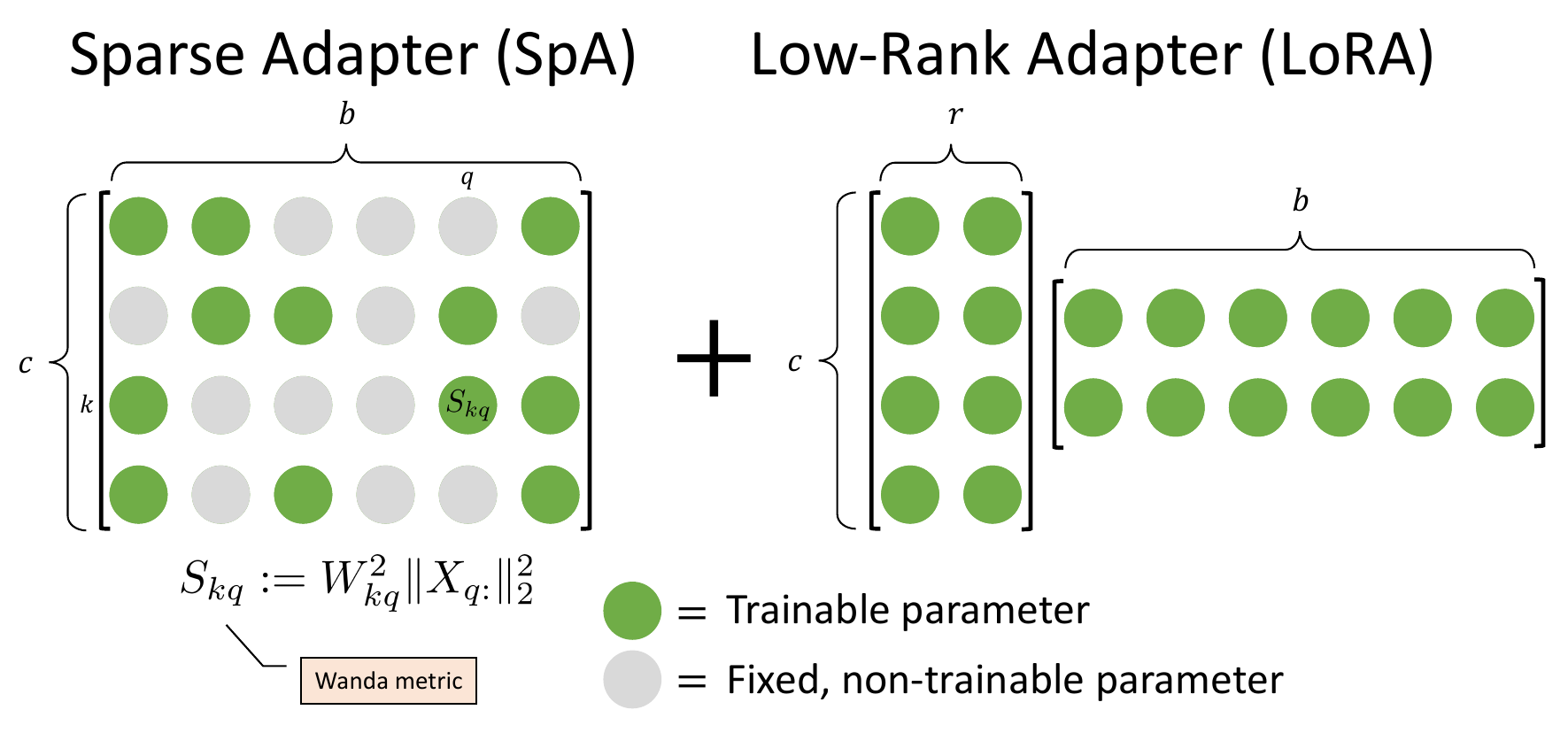}
    \captionsetup{width=\textwidth}
    \captionsetup{aboveskip=5pt, belowskip=0pt}
    \caption{\algname{Supra} combines a sparse adapter based on \algname{Super} weights with a low-rank adapter. To select the parameters for training in the sparse adapter, we employ the \algname{Wanda} metric (Equation~\ref{eq:wanda_metric}). The rank $r$ of the low-rank adapter is determined from the matched parameter-count budget using Equation~\ref{eq:lora_adaptive_r}.}
    \label{fig:supra_as_a_sum_of_spa_and_lora}
\end{figure*}

\algname{Supra} is a hybrid fine-tuning strategy that combines two complementary approaches: a sparse adapter selected via the \algname{Super} strategy and a low-rank adapter based on \algname{LoRA}. This design belongs to the same broad family as hybrid sparse--low-rank methods such as \algname{RoSA}~\citep{nikdan2024rosa}; the main distinction is that our sparse support is selected once using a training-free score, and we explicitly compare \algname{Wanda}-style and magnitude-only BottomK supports under a matched scalar-parameter budget.

For layer \(i\), \algname{Supra} uses the adapted weight
\begin{equation}
    \widehat W^{(i)}
    =
    W^{(i)}
    +
    M^{(i)} \odot U^{(i)}
    +
    \gamma_i L^{(i)} R^{(i)},
    \label{eq:supra_update}
\end{equation}
where \(M^{(i)}=\psi_{X^{(i)}}(W^{(i)},s_i,\rho)\), \(U^{(i)}\) is trainable only on the support of \(M^{(i)}\), \(L^{(i)} \in \R^{c_i \times r_i}\), \(R^{(i)} \in \R^{r_i \times b_i}\), and \(\gamma_i\) is the low-rank scaling factor. The selection direction \(\rho\) gives the TopK or BottomK sparse-support variant, and the pretrained matrix \(W^{(i)}\) is otherwise frozen.

The \algname{Supra} training objective optimizes only the adapter variables:
\begin{equation}
    \min_{\{U^{(i)},L^{(i)},R^{(i)}\}_{i\in\mathcal A}}
    \mathcal{L}\!\left(\mathcal{D}; \{\widehat W^{(i)}\}_{i\in\mathcal A}, \{W^{(j)}\}_{j\notin\mathcal A}, \bar w\right),
    \label{eq:supra_objective}
\end{equation}
where \(\mathcal A\) is the adapted-layer index set and the pretrained model parameters are frozen.

The intuition behind this design stems from the observation that \algname{LoRA}, by construction, learns updates in a low-rank parameterization. Although the \algname{LoRA} adapter can be expanded into a full matrix, its implementation uses only \(r_i(c_i+b_i)\) trainable scalar parameters. This evolving low-rank component lets the model adjust task-specific directions during fine-tuning.

In contrast, sparse adaptation by \algname{Super} selects a fixed subset of parameters to train. These parameters are chosen based on their estimated layer-wise influence under the \algname{Wanda}-style reconstruction proxy. Once selected, this subset remains constant during training, resulting in updates restricted to a fixed support.

By combining these two adapters, \algname{Supra} balances two update types: the sparse adapter directly optimizes a fixed \algname{Wanda}-ordered subset of pretrained weights, while the \algname{LoRA} component adapts to the target task by learning task-specific directions in the parameter space. In essence:

\begin{itemize}
    \item The sparse adapter focuses training on a fixed subset of parameters ordered by the \algname{Wanda}-style reconstruction proxy.
    \item The low-rank adapter enables \textit{task-specific adaptation} by learning new subspaces of importance in a structured and efficient manner.
\end{itemize}

This combination lets \algname{Supra} interpolate between sparse and low-rank adaptation under explicit parameter-count budgets. The explicit control over the parameter budget and its division between sparse and low-rank components enables practitioners to tailor the method to different resource constraints or adaptation scenarios.

\subsection{Budget-Matched Rank Conversion in \texorpdfstring{\algname{Supra}}{Supra}}

\algname{Supra} determines the low-rank adapter rank \( r_i \) from a per-layer parameter-count budget \(T_i\). Using \(\lfloor\cdot\rfloor\) for the floor function, this paper uses two budget conventions:

\begin{align}
    T_i^{\mathrm{rank}}(r_0) &= r_0(c_i+b_i),
    \label{eq:rank_equivalent_budget}\\
    T_i^{\mathrm{density}}(p) &= \lfloor p c_i b_i \rfloor.
    \label{eq:density_equivalent_budget}
\end{align}
The first is the \emph{\algname{LoRA}-rank-equivalent} budget: it gives \algname{Super} or \algname{Supra} the same number of trainable scalar parameters that a rank-\(r_0\) \algname{LoRA} adapter would use in that layer. This is the budget used in all reported experiments so that \algname{Super} and \algname{Supra} match the corresponding \algname{LoRA} baseline by number of trainable scalar parameters. The second is a density-equivalent sparse budget; it is useful conceptually but is not the convention used for the main comparisons. A fixed density \(p\) and a fixed rank-equivalent budget are generally different unless the implied density is layer-dependent:
\begin{equation}
    p_i(r_0)=\frac{r_0(c_i+b_i)}{c_i b_i}.
    \label{eq:rank_equivalent_density}
\end{equation}

To divide a chosen budget \(T_i\) between the low-rank and sparse adapters, we introduce \(\lambda \in [0,1]\), the \textit{low-rank budget fraction}. We use \(\lambda\) rather than \(\alpha\) to avoid confusion with the standard \algname{LoRA} scaling parameter \(\alpha_{\mathrm{LoRA}}\).

\begin{itemize}
    \item If \( \lambda = 1 \), the target budget is assigned to the low-rank component. Under \(T_i=T_i^{\mathrm{rank}}(r_0)\), this recovers rank-\(r_0\) \algname{LoRA}.
    \item If \( \lambda = 0 \), the full budget is allocated to the sparse adapter, resulting in the sparse-only \algname{Super} variant.
    \item If \( \lambda = 0.5 \), the target budget is approximately split between low-rank and sparse adapters.
\end{itemize}

Under this scheme, the low-rank adapter receives approximately \( \lambda T_i \) trainable scalar parameters, while the sparse adapter is assigned the remaining parameter-count budget.

Unlike sparse adaptation, which directly operates on a fixed number of parameters, low-rank adaptation requires specifying a rank \( r_i \). To bridge this gap, we define the following budget-matched rule for selecting \( r_i \) based on \(\lambda\) and \(T_i\):
\begin{equation}
    \label{eq:lora_adaptive_r}
    r_i = \left\lfloor \frac{\lambda T_i}{c_i + b_i} \right\rfloor.
\end{equation}

This formulation ensures that the low-rank adapter uses
\begin{equation}
    l_i = r_i (c_i + b_i)
\end{equation}
parameters, and the sparse adapter is assigned the remaining
\begin{equation}
    s_i = T_i - l_i
\end{equation}
parameters. Therefore \(s_i \geq 0\), and the total number of trainable parameters remains fixed: \(s_i + l_i = T_i\).

Under the rank-equivalent budget used in the main experiments, \(T_i=T_i^{\mathrm{rank}}(r_0)=r_0(c_i+b_i)\), so the conversion simplifies to
\begin{equation}
    r_i=\lfloor \lambda r_0\rfloor,
    \qquad
    s_i = \bigl(r_0-\lfloor \lambda r_0\rfloor\bigr)(c_i+b_i).
    \label{eq:rank_equivalent_simplification}
\end{equation}

Due to integer rank constraints in \eqref{eq:lora_adaptive_r}, the realized low-rank ratio may deviate slightly from the target value \( \lambda \). In particular, \(l_i/(l_i+s_i)\) need not equal \( \lambda \). This discrepancy arises from the floor operation, which ensures that \( r_i \) remains an integer and that the low-rank component never exceeds the assigned budget. Under a density budget \(T_i^{\mathrm{density}}(p)\), \(\lambda=1\) may leave a sparse remainder unless that remainder is explicitly discarded; under the rank-equivalent budget \(T_i^{\mathrm{rank}}(r_0)\), \(\lambda=1\) gives \(s_i=0\) and recovers rank-\(r_0\) \algname{LoRA}.

\section{Experiments}
\label{sec:experiments}

\subsection{Experimental Setup}

\paragraph{Datasets}

The main experiments use the \textbf{Math17K} instruction-tuning dataset. We split the fine-tuning data into training and validation subsets and use the held-out validation split only for learning-rate and checkpoint selection.

To the best of our knowledge and according to the dataset construction used in our experiments, we use Math17K rather than the earlier Math10K setup because Math10K contained examples overlapping with some of our evaluation benchmarks. The Math17K split used in this paper excludes examples from AddSub, MultiArith, SingleEq, GSM8K, AQuA, and SVAMP, so the evaluation benchmarks are not part of the fine-tuning data.

We evaluate on six arithmetic-reasoning benchmarks: AddSub~\citep{hosseini2014learning}, MultiArith~\citep{roy2016solving}, SingleEq~\citep{koncel2015parsing}, GSM8K~\citep{cobbe2021training}, AQuA~\citep{ling2017program}, and SVAMP~\citep{patel2021nlp}. Thus the main protocol is task-specific arithmetic fine-tuning followed by arithmetic evaluation.

\paragraph{Training Details}

Except for the full-parameter reference rows, all reported trainable methods use adapter-based fine-tuning, where the pretrained model remains frozen and only adapter parameters are optimized. The reported experiments use two Llama-family models: \texttt{Llama-3.2-1B} and \texttt{Meta-Llama-3-8B}. For adapter methods, we adapt all standard linear projections in each Llama block:
\[
\texttt{q\_proj},\ \texttt{k\_proj},\ \texttt{v\_proj},\ \texttt{o\_proj},\ \texttt{gate\_proj},\ \texttt{up\_proj},\ \texttt{down\_proj}.
\]
This is the full set of linear modules adapted in all reported adapter experiments.

\textbf{Compared methods.} We compare \algname{LoRA}, \algname{SIFT} (TopK), \algname{SIFT} (RandK), \algname{RoSA}, \algname{Super} (TopK), \algname{Super} (BottomK), \algname{Super} (RandK), \algname{Magnitude} (TopK/BottomK), and \algname{Supra} with \( \lambda \in \{0.3,0.5,0.8\} \) under both \algname{Wanda}-selected and magnitude-selected sparse supports. For \algname{Supra}, \(\lambda\) is the low-rank budget fraction defined in \Cref{sec:supra}; in method names, TopK denotes selecting the largest scores and BottomK denotes selecting the smallest scores under the same budget. The \algname{Magnitude} rows are PaFi-style magnitude-mask sparse fine-tuning baselines using the score \(|W_{ij}|\), while \algname{Super} and the default \algname{Supra} rows use the activation-weighted \algname{Wanda} score. The Math17K one- and three-epoch variant grids for these methods, including accuracy and perplexity, are reported in the appendix; the hybrid TopK--BottomK ablation is collected separately in \Cref{sec:hybrid_super_appendix}. Each table row also lists the validation-selected learning rate.

\textbf{Parameter budgets.} The main run uses the rank-equivalent budget \(r_0=8\). For a baseline \algname{LoRA} rank \(r_0\), the target per-layer parameter-count budget is \(T_i^{\mathrm{rank}}(r_0)=r_0(c_i+b_i)\), matching the number of trainable scalar parameters in rank-\(r_0\) \algname{LoRA}. Sparse methods train the same number of scalar entries in the corresponding layer, and \algname{Supra} splits the same numerical budget between sparse and low-rank components. Due to integer ranks and method-specific implementation details, the realized trainable-parameter counts may differ slightly, but all adapter methods use approximately the same model-specific rank-8 budget: \(5.6\)M trainable parameters for \texttt{Llama-3.2-1B} and \(21.0\)M for \texttt{Meta-Llama-3-8B}. Our comparisons match trainable scalar counts across adapter methods. Implementation-level memory, checkpoint-size, calibration-time, and throughput measurements are reported separately in \Cref{sec:efficiency_measurements}; matching trainable scalar counts alone does not imply identical runtime behavior.

\textbf{Learning-rate selection.} Because different PEFT parameterizations can have different optimization scales, we tune the learning rate independently for each method using the same candidate grid:
\[
5\cdot 10^{-5},\ 10^{-4},\ 5\cdot 10^{-4},\ 10^{-3},\ 5\cdot 10^{-3},\ 10^{-2},\ 5\cdot 10^{-2},\ 10^{-1}.
\]
Each method is trained at all candidate learning rates, and a single learning rate is selected for that method using the held-out Math17K validation negative log-likelihood. The reported test accuracies and perplexities for a method all come from this one validation-selected learning rate; we do not select learning rates using test benchmark accuracy.

\textbf{Optimization and calibration.} In the Math17K comparisons, we fine-tune each candidate run for either one or three epochs, with batch size 16, maximum sequence length 256, and 100 warmup steps. Training uses Adam with weight decay 0 under the standard causal language-modeling objective. Unless stated otherwise, \algname{LoRA} uses \( \alpha_{\mathrm{LoRA}}=16 \), dropout 0.05, and no bias training. The compact main-text tables report, for each method, the better of its one- and three-epoch runs by average benchmark accuracy; the appendix keeps the schedules separate. For \algname{Super} and the default \algname{Supra} variants, the \algname{Wanda}-style sparse masks are computed from 128 C4 calibration samples. The magnitude-only variants do not use a calibration pass; they rank candidate sparse coordinates by \(|W_{ij}|\). When a sparse support is used, it is fixed before fine-tuning and only the selected sparse entries are optimized during Math17K fine-tuning. \Cref{fig:math17k_optimization_curves} shows the training and held-out validation loss curves for the \texttt{Meta-Llama-3-8B} three-epoch Math17K runs, illustrating the overall optimization behavior under this protocol and the validation-loss signal used for learning-rate selection. To check whether in-domain input statistics materially change the selected sparse support, we also ran a targeted math-calibrated ablation for \texttt{Llama-3.2-1B} on Math7K: the fine-tuning setup is fixed, while the sparse-mask calibration source is changed from C4 to Math7K. The per-benchmark results in \Cref{tab:math7k_calibration_source_ablation} show that this calibration change has only a small effect in that one-epoch setting.

\begin{table}[htbp]
\centering
\scriptsize
\caption{Calibration-source ablation for sparse-only \algname{Super}. Model: \texttt{Llama-3.2-1B}; fine-tuning dataset: Math7K; training budget: one epoch; rank-equivalent budget: \(r_0=8\). The Base row is frozen. C4 rows use the default C4 calibration samples for the \algname{Wanda}-style score; Math7K rows compute the same score on Math7K calibration samples. Trainable rows use approximately \(5.6\)M trainable parameters and the validation-selected learning rate. Results are exact-answer accuracy (\%).}
\label{tab:math7k_calibration_source_ablation}
\resizebox{\textwidth}{!}{%
\begin{tabular}{lccccccccc}
\toprule
\textbf{Method} & \textbf{Calibration} & \textbf{Selected LR} & \textbf{AddSub} & \textbf{MultiArith} & \textbf{SingleEq} & \textbf{GSM8K} & \textbf{AQuA} & \textbf{SVAMP} & \textbf{Average} \\
\midrule
\textsc{Base} (frozen) & -- & -- & 13.67 & 4.67 & 21.46 & 2.81 & 21.65 & 11.60 & 12.64 \\
\midrule
\algname{Super} (TopK) & C4 & $5\cdot10^{-4}$ & 35.95 & 63.83 & 48.62 & 21.46 & 21.65 & 29.00 & 36.75 \\
\algname{Super} (TopK) & Math7K & $5\cdot10^{-4}$ & 38.23 & 61.83 & 48.62 & 22.06 & 20.08 & 29.60 & 36.74 \\
\algname{Super} (BottomK) & C4 & $10^{-3}$ & 40.76 & 59.33 & 52.56 & 21.76 & 17.32 & 32.30 & \textbf{37.34} \\
\algname{Super} (BottomK) & Math7K & $10^{-3}$ & 41.01 & 57.83 & 51.38 & 20.17 & 20.08 & 32.20 & 37.11 \\
\bottomrule
\end{tabular}
}
\end{table}

\begin{figure}[htbp]
\centering
\begin{minipage}{0.49\textwidth}
\centering
\includegraphics[width=\linewidth]{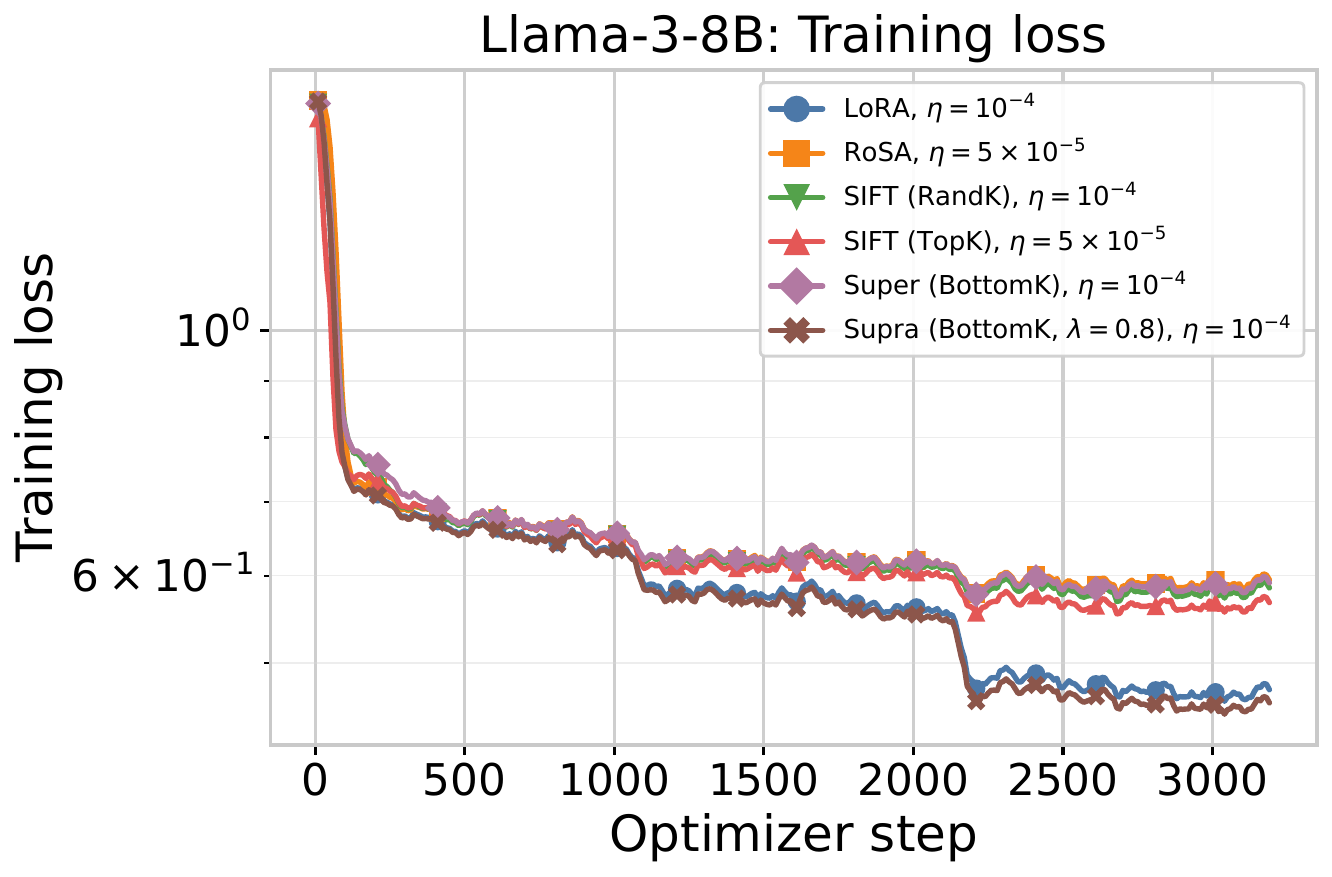}
\end{minipage}
\hfill
\begin{minipage}{0.49\textwidth}
\centering
\includegraphics[width=\linewidth]{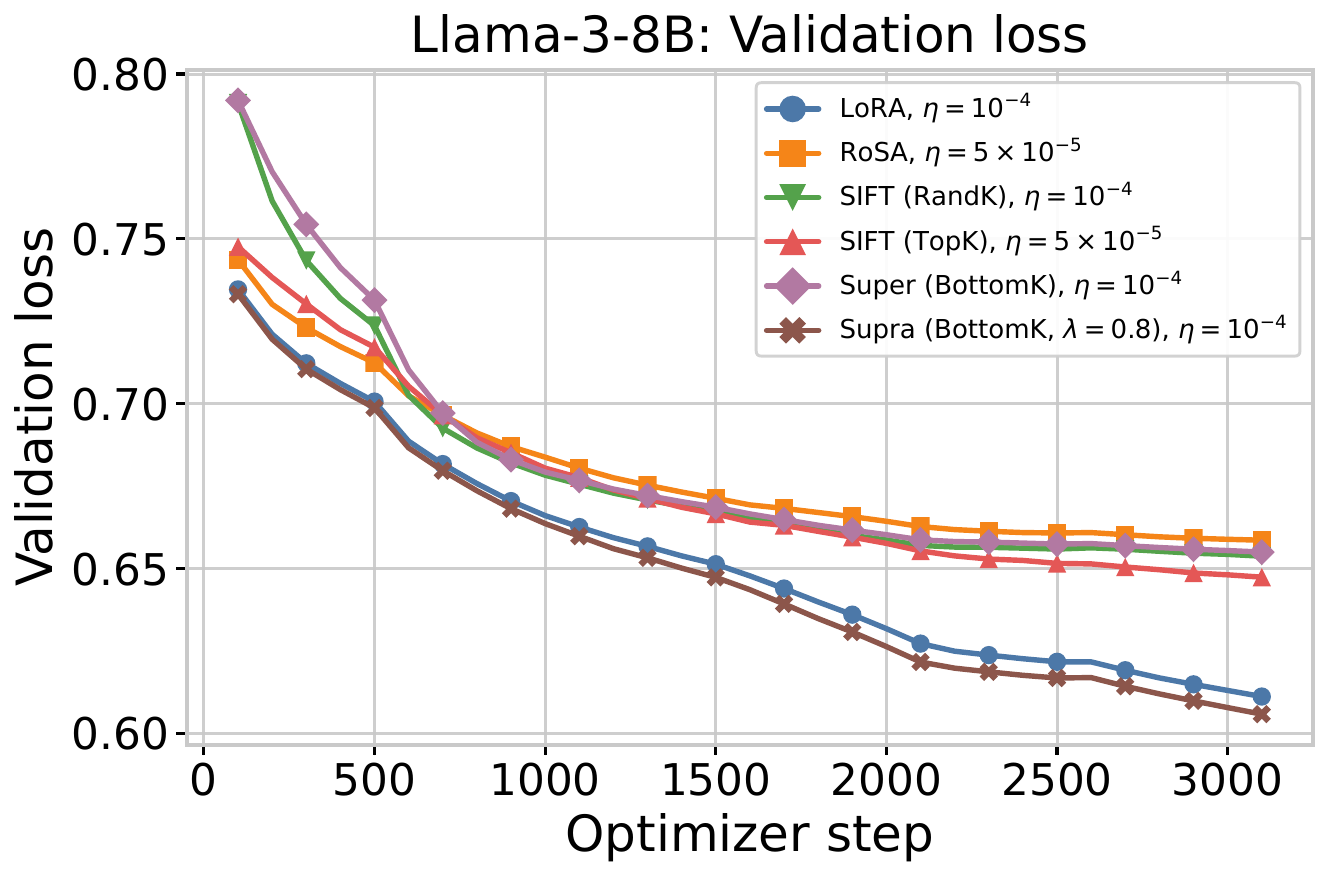}
\end{minipage}
\caption{Optimization curves for \texttt{Meta-Llama-3-8B} Math17K fine-tuning over three epochs. Left: training loss, plotted with a log-scaled \(y\)-axis. Right: held-out Math17K validation loss. Each curve uses the validation-selected learning rate for the corresponding method under the rank-equivalent \(r_0=8\) budget.}
\label{fig:math17k_optimization_curves}
\end{figure}

\textbf{Evaluation metric.} We report exact-answer accuracy on each math benchmark and the arithmetic mean over benchmarks. For generation-based accuracy, we use the arithmetic evaluation pipeline to extract either the numeric answer or the answer option letter for AQuA. Perplexity is reported only as an auxiliary language-modeling diagnostic and is not used for the main arithmetic conclusion. Exact-answer accuracy is the primary metric. Lower perplexity does not necessarily imply higher benchmark accuracy because the language-modeling loss can be dominated by non-answer tokens, while the benchmark metric depends only on the extracted final numeric answer or multiple-choice option. In our evaluation pipeline, perplexity is computed over the gold target text appended after the prompt, with prompt tokens masked from the loss; the target may include rationale or answer text depending on the dataset record, plus the end-of-sequence token. Per-dataset perplexity is the exponentiated token-weighted mean negative log-likelihood over up to 120 examples per benchmark, and the reported average is an aggregate diagnostic over the same benchmark set.

\textbf{Hardware and reproducibility.} The main experiment is run on the KAUST ORIX cluster using 4 NVIDIA H200 GPUs with 140GB of memory per GPU. The implementation is available at \url{https://github.com/vectozavr/SuperTuning}. The launch pipeline exposes the random seed, logs all adapter runs to Weights \& Biases, records the selected learning rate, and writes machine-readable result files for table construction. Repeating the full benchmark suite across multiple random seeds would be computationally expensive, so the reported tables use a single controlled seed (seed 0); estimating multi-seed uncertainty is left for future work.

\subsection{Efficiency Measurements}
\label{sec:efficiency_measurements}

Although all adapter methods are matched by trainable scalar count, this does not automatically imply identical memory or runtime behavior. We therefore report implementation-level efficiency measurements for representative Math17K profiling runs in \Cref{tab:efficiency_measurements}. The profiling setup uses the same models, adapted modules, rank-equivalent budget \(r_0=8\), batch size 16, micro-batch size 16, maximum sequence length 256, C4 calibration source for \algname{Wanda}-style masks, and NVIDIA H200 hardware as the main experiments. To keep the measurement inexpensive, each row runs 20 optimizer steps with learning rate \(5\cdot10^{-4}\); these runs are used only for efficiency profiling and do not add accuracy baselines.

These measurements should be interpreted as costs of our current implementation rather than theoretical sparse-kernel speedups. The optimizer-state column reports measured Adam tensor storage, not a normalized two-FP32-state estimate: for full fine-tuning the model parameters are bf16 in these runs, so the two Adam moment tensors occupy approximately the same memory as one FP32 tensor per trainable parameter. The \algname{Super}, \algname{Magnitude}, \algname{Supra}, and \algname{Supra-Mag} implementations store sparse trainable values and integer index buffers, so their measured optimizer-state storage is sparse. However, the current sparse linear layer forms a dense effective weight for the matrix multiply and computes a dense gradient matrix before gathering the selected sparse coordinates. Thus, sparse trainable-parameter counts reduce adapter checkpoint size and measured optimizer-state storage, but do not by themselves guarantee wall-clock speedups. \algname{SIFT} also stores sparse trainable values, but its selection/update path requires dense gradients; this is reflected in its peak-memory measurements. Calibration time is reported only for methods with a separate C4 mask-construction pass; methods whose sparse support is magnitude-only, low-rank-only, full fine-tuning, or selected inside training are marked ``--''.

\begin{table}[!t]
\centering
\scriptsize
\caption{Implementation-level efficiency measurements for Math17K profiling runs. Both models use rank-equivalent budget $r_0=8$, batch size 16, micro-batch size 16, and 20 optimizer steps on a single NVIDIA H200 GPU. Checkpoint size is the saved full-model checkpoint for full fine-tuning and the saved adapter state otherwise. Peak memory is peak CUDA allocated memory during the measured training region. Adam state is measured directly from tensor storage present in the instantiated optimizer after the profiling steps; it reflects the actual parameter dtypes and optimizer-state representation used by each method rather than a uniform two-FP32-state estimate. In particular, full fine-tuning stores two bf16 Adam moment tensors in these runs, so its state occupies approximately the same memory as one FP32 tensor per trainable parameter.}
\label{tab:efficiency_measurements}
\resizebox{\textwidth}{!}{%
\begin{tabular}{clrrrrrrrr}
\toprule
\textbf{Model} & \textbf{Method} & \textbf{Trainable Params} & \textbf{Ckpt. Size} & \textbf{Calib. Time} & \textbf{Peak Mem.} & \textbf{Adam State} & \textbf{Steps/s} & \textbf{Tokens/s} & \textbf{Wall Time} \\
\midrule
\multirow{8}{*}{\texttt{Llama-3.2-1B}} & Full FT & 1.24B & 2.32 GiB & -- & 20.88 GiB & 4.60 GiB & 5.688 & 17819 & 3.5 s \\
 & \algname{RoSA} & 5.64M & 5.4 MiB & -- & 17.98 GiB & 10.8 MiB & 4.368 & 13685 & 4.6 s \\
 & \algname{SIFT} (TopK) & 5.64M & 107.6 MiB & -- & 21.28 GiB & 43.0 MiB & 7.073 & 22160 & 2.8 s \\
 & \algname{LoRA} & 5.64M & 21.6 MiB & -- & 18.04 GiB & 43.0 MiB & 5.808 & 18197 & 3.4 s \\
 & \algname{Super} (BottomK) & 5.64M & 43.1 MiB & 2.8 s & 15.81 GiB & 43.0 MiB & 7.933 & 24855 & 2.5 s \\
 & \algname{Magnitude} (BottomK) & 5.64M & 43.1 MiB & -- & 15.81 GiB & 43.0 MiB & 8.287 & 25962 & 2.4 s \\
 & \algname{Supra} (BottomK, $\lambda=0.8$) & 5.64M & 43.1 MiB & 2.8 s & 19.55 GiB & 43.0 MiB & 5.831 & 18269 & 3.4 s \\
 & \algname{Supra-Mag} (BottomK, $\lambda=0.3$) & 5.64M & 43.1 MiB & -- & 19.54 GiB & 43.0 MiB & 5.743 & 17993 & 3.5 s \\
\midrule
\multirow{8}{*}{\texttt{Meta-Llama-3-8B}} & Full FT & 8.03B & 14.97 GiB & -- & 77.91 GiB & 29.92 GiB & 1.706 & 5344 & 11.7 s \\
 & \algname{RoSA} & 20.97M & 20.1 MiB & -- & 54.72 GiB & 40.0 MiB & 1.850 & 5795 & 10.8 s \\
 & \algname{SIFT} (TopK) & 20.97M & 400.1 MiB & -- & 78.50 GiB & 160.0 MiB & 1.874 & 5870 & 10.7 s \\
 & \algname{LoRA} & 20.97M & 80.2 MiB & -- & 54.90 GiB & 160.0 MiB & 2.044 & 6403 & 9.8 s \\
 & \algname{Super} (BottomK) & 20.97M & 160.1 MiB & 11.3 s & 46.22 GiB & 160.0 MiB & 2.252 & 7055 & 8.9 s \\
 & \algname{Magnitude} (BottomK) & 20.97M & 160.1 MiB & -- & 46.22 GiB & 160.0 MiB & 2.249 & 7046 & 8.9 s \\
 & \algname{Supra} (BottomK, $\lambda=0.8$) & 20.97M & 160.3 MiB & 11.2 s & 60.46 GiB & 160.0 MiB & 1.716 & 5377 & 11.7 s \\
 & \algname{Supra-Mag} (BottomK, $\lambda=0.3$) & 20.97M & 160.3 MiB & -- & 60.44 GiB & 160.0 MiB & 1.689 & 5291 & 11.8 s \\
\bottomrule
\end{tabular}
}
\end{table}

\subsection{Comparison with other fine-tuning methods}

This subsection discusses the rank-equivalent \(r_0=8\) arithmetic results shown earlier in \Cref{tab:main_text_llama1b_compact,tab:main_text_llama8b_compact}. Each adapter method uses approximately the same rank-equivalent trainable-parameter budget, and each reported row uses the learning rate selected by held-out Math17K validation loss for the corresponding epoch budget. The compact tables select the better one- or three-epoch run for each method by average benchmark accuracy; for \algname{Super} and \algname{Supra}, they also select the best observed variant from the corresponding support and budget-split grid. Because the schedule-selected compact rows for our proposed methods use BottomK supports, we use BottomK as the default support for \algname{Super} and \algname{Supra} unless stated otherwise. The complete per-benchmark accuracy and perplexity tables are reported in \Cref{sec:llama1b_math17k_1epoch_appendix,sec:llama1b_math17k_3epoch_appendix,sec:llama8b_math17k_1epoch_appendix,sec:llama8b_math17k_appendix}.

For the 1B model, \algname{Supra} (BottomK, \(\lambda=0.8\)) reaches 62.23\% average exact-answer accuracy, the strongest observed adapter result in the comparison. The strongest external baseline is \algname{LoRA} at 61.07\%, so \algname{Supra} achieves a higher observed average by 1.16 percentage points. The sparse-only \algname{Super} (BottomK) row reaches 55.46\%, below \algname{LoRA} but above the corresponding TopK sparse row in the full appendix table.

For the 8B model, the strongest schedule-selected result is obtained by \algname{Supra-Mag} (BottomK, \(\lambda=0.3\)), which reaches 79.12\% average accuracy. The magnitude-only sparse baseline is essentially tied at 79.02\%, while the strongest \algname{Wanda}-selected \algname{Supra} row reaches 78.66\%. Thus, the 8B results should be interpreted less as evidence that activation weighting is always superior and more as evidence that low-score pretrained-weight supports, especially low-magnitude supports, are highly effective in this setting. Combining such supports with \algname{LoRA} can preserve or slightly improve performance under the same scalar-parameter budget. Relative to external baselines, \algname{Supra-Mag} achieves a higher observed average than \algname{SIFT} (RandK) by 1.29 points and than \algname{LoRA} by 5.95 points.

Full fine-tuning is included as an unbudgeted full-parameter reference rather than as a PEFT baseline, since it updates all pretrained parameters. We do not interpret these full fine-tuning rows as fully optimized upper bounds; they are protocol-matched full-parameter reference runs using the same validation-based learning-rate selection framework. On \texttt{Llama-3.2-1B}, full fine-tuning reaches 59.77\% average accuracy after one epoch and 65.70\% after three epochs, while the best rank-8-equivalent adapter reaches 59.69\% and 62.23\% under the corresponding one- and three-epoch budgets. On the 8B model, several adapter configurations achieve higher average accuracy than the full-parameter reference. We therefore use full fine-tuning only as a reference point for the cost and behavior of updating all parameters under the same training setup.

\section{Discussion}
\label{sec:discussion}
The main empirical conclusion is that fixed sparse adaptation supports can be highly effective for parameter-efficient fine-tuning. Under the same rank-equivalent trainable-parameter budget and the same validation-based learning-rate selection protocol, the strongest sparse and sparse--low-rank variants achieve the best observed average arithmetic accuracy in the single-seed, schedule-selected compact comparisons on both tested model sizes. On the 1B model, \algname{Supra} (BottomK, \(\lambda=0.8\)) reaches 62.23\% average accuracy, a higher observed average than the strongest external baseline, \algname{LoRA}, by 1.16 percentage points. On the 8B model, \algname{Supra-Mag} (BottomK, \(\lambda=0.3\)) reaches 79.12\%, nearly tied with the magnitude-only sparse baseline at 79.02\% and 1.29 points above the strongest external baseline, \algname{SIFT} (RandK). These gains are obtained with simple fixed mask construction and without gradient or curvature-based mask search.

The strongest schedule-selected sparse supports usually come from the low-score end of the ordering. This pattern appears for both the \algname{Wanda}-style activation-weighted score and the PaFi-style magnitude-only score. However, the best direction is not universal: in some fixed-schedule sparse-only runs, TopK or mixed TopK--BottomK supports can outperform pure BottomK. We therefore view BottomK as a strong empirical default in our setting rather than as an optimal selection rule.

The hybrid \algname{Supra} results further suggest that sparse and low-rank adaptation can be complementary under a fixed parameter budget, but the best allocation is not universal. \algname{Supra} is strongest among the proposed methods on the 1B model, while \algname{Supra-Mag} is strongest on the 8B model and is essentially tied with the magnitude-only sparse baseline. This suggests that the best allocation between sparse updates and low-rank updates may depend on model scale, optimization budget, and task distribution, but the overall pattern is encouraging: sparse updates are competitive with standard PEFT baselines in these single-seed comparisons, and combining them with low-rank adapters can preserve or improve performance at the same trainable-parameter count.

Overall, our results support a practical view of \algname{Super} and \algname{Supra}: simple low-score orderings can identify useful sparse adaptation coordinates. A cheap activation-based calibration pass provides one such ordering through the \algname{Wanda}-style score, while magnitude-only ordering is a strong PaFi-style alternative. More broadly, the experiments suggest that pruning-inspired signals can be repurposed for fine-tuning.

\section{Limitations}
\label{sec:limitations}

Our study has several limitations that we hope future work will address.

\paragraph{Single-seed uncertainty.}
All main tables use a single controlled seed. This is common in expensive LLM fine-tuning studies but means that small differences between methods should be interpreted as observed single-seed gaps rather than statistically established improvements. Multi-seed uncertainty estimates are left for future work.

\paragraph{Schedule-selected summaries.}
The compact main-text tables select the better one- or three-epoch schedule for each method family and, for \algname{Super}/\algname{Supra}, the best reported variant. The appendix provides the fixed-schedule grids. The compact tables should therefore be interpreted as best-observed schedule-selected summaries rather than universal dominance claims.

\paragraph{Scope of evaluation.}  
The quantitative comparisons in this version focus on arithmetic-reasoning evaluation with \texttt{Llama-3.2-1B} and \texttt{Meta-Llama-3-8B} after Math17K fine-tuning. This is a limited slice of the broad application space of LLMs. In particular, 70B-scale models, other architecture families, and language-generation tasks that hinge on open-ended semantics (e.g., summarization, dialogue, or code synthesis) may stress different model components than symbolic math reasoning. The 8B results in \Cref{sec:llama8b_math17k_appendix} are useful as a larger-checkpoint check, but should not be read as a scaling law.

\paragraph{Dependence on the \algname{Wanda} saliency proxy.}  
Our new algorithm \algname{Super} (and hence the default \algname{Supra}) selects trainable weights using an activation-weighted magnitude heuristic inherited from the \algname{Wanda} pruning metric. Although this choice is fast and training-free, it is still an \emph{approximation} of true importance; weights that appear non-salient under the proxy may in fact become critical once the task distribution shifts or once other parameters are updated. The magnitude-only baselines show that removing activation weighting can also work well in some settings. We therefore do not claim that the current mask-selection rule is optimal, or even that a pruning-saliency score should always be used in the same direction for fine-tuning. Further ablations against activation-only masks, row-wise variants, alternative calibration sets, and more expressive but costlier criteria (e.g.\ curvature-aware saliency or gradient-flow statistics) remain important.

\paragraph{Fixed sparse support.}  
The sparse component of \algname{Supra} is \emph{static} throughout fine-tuning: once a weight is deemed non-trainable, it can never be activated.  Dynamic sparsity---allowing the mask to grow, prune, or re-allocate budget during training---has shown promise in other contexts and may close the gap to full fine-tuning in edge cases where the initial selection is sub-optimal.

\paragraph{System efficiency.}
The profiling results in \Cref{sec:efficiency_measurements} measure our current implementation, not an optimized sparse-kernel system. They use short fixed-step runs, so they are useful for comparing implementation-level memory, checkpoint size, calibration cost, and local throughput under the same setup, but they should not be interpreted as asymptotic runtime claims. Additional sparse-kernel engineering, sparse optimizer implementations, or different batching regimes could change the wall-clock profile.

\paragraph{Layer-local parameter budgets.}  
Our budget-matched rank rule (\Cref{eq:lora_adaptive_r}) applies the same low-rank budget fraction \(\lambda\) independently to every adapted linear layer. In practice, different layers contribute unequally to downstream task performance; early attention blocks, for example, often admit more aggressive compression than middle MLP layers. A global, data-driven re-allocation of the parameter budget---akin to neural architecture search---might improve efficiency further.

\begin{ack}
    The research reported in this publication was supported by funding from King Abdullah University of Science and Technology (KAUST): i) KAUST Baseline Research Scheme, ii) CRG Grant ORFS-CRG12-2024-6460, and iii) Center of Excellence for Generative AI, under award number 5940.
\end{ack}

\bibliographystyle{plainnat}
\bibliography{example_paper}

\newpage
\appendix


\section{Table of Frequently Used Notation} \label{sec:notation}

\begin{table}[H]
\caption{Notation table}
\label{tab:notation_table}
\centering
\small
\begin{tabularx}{\textwidth}{@{}>{\raggedright\arraybackslash}p{0.19\textwidth}@{\hspace{0.5em}}c@{\hspace{0.5em}}X@{}}
\toprule
    \(\mathcal{N}\) & -- & Pretrained language model. \\
    \(m\) & -- & Number of candidate fully connected matrices considered for adaptation. \\
    \(\mathcal{W}=\{W^{(i)}\}_{i=1}^m\) & -- & Candidate fully connected weight matrices considered for adaptation. \\
    \(W^{(i)}\) & -- & Weight matrix \(i\), \(W^{(i)} \in \R^{c_i \times b_i}\), where \(c_i\) and \(b_i\) are output and input dimensions. \\
    \(\bar{w}\), \(\bar{\delta}\), \(\bar d\) & -- & Remaining model parameters \(\bar w\in\R^{\bar d}\), their optional perturbation \(\bar\delta\in\R^{\bar d}\), and their dimension. \\
    \(\mathcal{D}\), \(\mathcal{L}\) & -- & Fine-tuning dataset and training loss. \\
    \(\mathcal{D}_{\mathrm{cal}}\) & -- & Tokenized calibration set used to collect activations for sparse-mask construction. \\
    \(\mathcal A\) & -- & Adapted-layer index set, \(\mathcal A\subseteq\{1,\ldots,m\}\). \\
    \(\Delta\), \(\Delta^{(i)}\) & -- & Collection of additive perturbations and the perturbation applied to \(W^{(i)}\). \\
    \(\mathcal{C}(\Delta,\bar{\delta})\) & -- & Constraint set defining the adapter structure, such as low-rank or sparse updates. \\
    \(\widehat W^{(i)}\) & -- & Adapted weight matrix after adding sparse and/or low-rank updates. \\
    \(X\), \(X^{(i)}\) & -- & Calibration activations for a generic linear layer or layer \(i\); \(X \in \R^{b \times a}\), where \(a\) is the number of calibration token positions. \\
    \(X_{j:}\), \(X_{q:}\), \(X_{r:}\) & -- & Rows of \(X\); \(j,q,r\) denote input-coordinate indices. \\
    \(W_{ij}\), \(W_{kq}\), \(W^{(i)}_{kq}\) & -- & Entries of a generic weight matrix or of layer \(i\); \(k,q\) denote row and column indices in the \algname{Wanda} score. \\
    \(G_{ij}\) & -- & Gradient-magnitude score for sparse selection, \(G_{ij}=|\partial \mathcal L/\partial W_{ij}|\). \\
    \(B_{ij}\) & -- & PaFi-style pretrained-magnitude score, \(B_{ij}=|W_{ij}|\). \\
    \(A_{ij}\) & -- & Activation-weighted magnitude score used for mask selection, \(A_{ij}=|W_{ij}|\|X_{j:}\|_2\). \\
    \(S_{kq}\) & -- & \algname{Wanda} saliency score for entry \(W_{kq}\). \\
    \(\rho\) & -- & Sparse-support selection direction, \(\rho\in\{\mathrm{TopK},\mathrm{BottomK}\}\). \\
    \(\psi_X(W,s,\rho)\) & \(\coloneqq\) & Binary mask in \(\{0,1\}^{c\times b}\) selecting either the global top \(s\) or global bottom \(s\) entries of \(W\), according to \(\rho\), ordered by \(A_{ij}\). \\
    \(M^{(i)}\) & -- & Binary sparse mask for layer \(i\), \(M^{(i)}=\psi_{X^{(i)}}(W^{(i)},s_i,\rho)\). \\
    \(U^{(i)}\) & -- & Sparse trainable delta matrix used through \(M^{(i)}\odot U^{(i)}\). \\
    \(L^{(i)}, R^{(i)}\) & -- & \algname{LoRA} factors, \(L^{(i)} \in \R^{c_i \times r_i}\) and \(R^{(i)} \in \R^{r_i \times b_i}\). \\
    \(\gamma_i\) & -- & \algname{LoRA} scaling factor for layer \(i\). \\
    \(p\) & -- & Sparse density, \(p \in (0,1]\). \\
    \(T_i\) & -- & Parameter-count budget assigned to layer \(i\). \\
    \(T_i^{\mathrm{rank}}(r_0)\) & -- & \algname{LoRA}-rank-equivalent budget, \(r_0(c_i+b_i)\). \\
    \(T_i^{\mathrm{density}}(p)\) & -- & Density-equivalent budget, \(\lfloor p c_i b_i\rfloor\). \\
    \(r_0\) & -- & Reference \algname{LoRA} rank used to define the rank-equivalent budget. \\
    \(p_i(r_0)\) & -- & Layerwise density induced by the rank-equivalent budget, \(r_0(c_i+b_i)/(c_i b_i)\). \\
    \(\lambda\) & -- & Low-rank budget fraction for \algname{Supra}, \(\lambda \in [0,1]\). \\
    \(r_i\) & -- & Low-rank adapter rank for layer \(i\). \\
    \(l_i\) & -- & Number of low-rank adapter parameters in layer \(i\), \(l_i=r_i(c_i+b_i)\). \\
    \(s_i\) & -- & Number of sparse-adapter parameters in layer \(i\), \(s_i=T_i-l_i\). \\
    \(e_k\) & -- & \(k^{\text{th}}\) standard basis vector. \\
    \(\|\cdot\|_2\) & -- & Euclidean norm. \\
    \(\|\cdot\|_0\) & -- & Number of nonzero entries. \\
    \(\|\cdot\|_F\) & -- & Frobenius norm. \\
    \(\langle x,y\rangle\) & -- & Euclidean inner product of vectors \(x\) and \(y\). \\
    \(\odot\) & -- & Element-wise product of matrices with the same shape (Hadamard product). \\
    \(\lfloor x\rfloor\) & -- & Floor function, the greatest integer less than or equal to \(x\). \\
\bottomrule
\end{tabularx}
\end{table}

\clearpage

\section{\texorpdfstring{\algname{Wanda}}{Wanda} metric for generating fine-tuning mask}
\label{sec:wanda_appendix}

For fixed calibration activations \(X\in\R^{b\times a}\), where \(a\) is the number of calibration token positions, we define the mask-generation mapping
\[
    \psi_X:\R^{c\times b}\times\{0,\ldots,cb\}\times\{\mathrm{TopK},\mathrm{BottomK}\}
    \to \{0,1\}^{c\times b}.
\]
It takes a weight matrix \(W\in\R^{c\times b}\), a target number of sparse entries \(s\in\{0,\ldots,cb\}\), and a selection direction \(\rho\in\{\mathrm{TopK},\mathrm{BottomK}\}\). For each weight \(W_{ij}\), it evaluates the metric \(A_{ij}=|W_{ij}|\|X_{j:}\|_2\). The TopK variant selects the global top \(s\) entries in the layer, while the BottomK variant selects the global bottom \(s\) entries. Since squaring is monotone on nonnegative scores, the same ordering is obtained from the squared \algname{Wanda} score in \Cref{eq:wanda_metric}. The output is a binary mask with ones at these selected positions:
\begin{equation}
\label{eq:psi_def_appendix}
\psi_X(W,s,\rho)_{ij}
\coloneqq
\begin{cases}
1, & (i,j)\text{ is among the }s\text{ largest entries of }A,\quad \rho=\mathrm{TopK},\\
1, & (i,j)\text{ is among the }s\text{ smallest entries of }A,\quad \rho=\mathrm{BottomK},\\
0, & \text{otherwise.}
\end{cases}
\end{equation}

For example, let us have a weight matrix 
\[
W = 
\begin{pmatrix}
    3 & -2 \\
    -2 & 4 \\
    1 & -6 \\
\end{pmatrix},
\]
and an input matrix 
\[
X = 
\begin{pmatrix}
    4 & 3 \\
    0 & 1 \\
\end{pmatrix},
\] now we compute the metric $|W_{ij}|\|X_{j:}\|_2$ for every entry of $W$:
\[
\begin{pmatrix}
    |W_{11}|\|X_{1:}\|_2 & |W_{12}|\|X_{2:}\|_2 \\
    |W_{21}|\|X_{1:}\|_2 & |W_{22}|\|X_{2:}\|_2 \\
    |W_{31}|\|X_{1:}\|_2 & |W_{32}|\|X_{2:}\|_2 \\
\end{pmatrix} = 
\begin{pmatrix}
    3\cdot 5  & 2\cdot 1 \\
    2\cdot 5  & 4\cdot 1 \\
    1\cdot 5  & 6\cdot 1 \\
\end{pmatrix} = 
\begin{pmatrix}
    15  & 2 \\
    10  & 4 \\
    5  & 6 \\
\end{pmatrix},
\] then the TopK masks are
\[
\psi_X(W, 1,\mathrm{TopK}) = 
\begin{pmatrix}
    1  & 0 \\
    0  & 0 \\
    0  & 0 \\
\end{pmatrix}, \quad
\psi_X(W, 2,\mathrm{TopK}) = 
\begin{pmatrix}
    1  & 0 \\
    1  & 0 \\
    0  & 0 \\
\end{pmatrix}, \quad
\psi_X(W, 4,\mathrm{TopK}) = 
\begin{pmatrix}
    1  & 0 \\
    1  & 0 \\
    1  & 1 \\
\end{pmatrix}.
\]
For the same scores, the BottomK variant instead selects the lowest-scoring entries; for example,
\[
\psi_X(W, 2,\mathrm{BottomK}) =
\begin{pmatrix}
    0  & 1 \\
    0  & 1 \\
    0  & 0 \\
\end{pmatrix}.
\]

\clearpage

\section{Additional experiments}
\label{sec:additional_experiments}

\subsection{Hybrid TopK--BottomK \texorpdfstring{\algname{Super}}{Super}}
\label{sec:hybrid_super_appendix}

The hybrid \algname{Super} variant tests whether a sparse support should be selected entirely from one end of the \algname{Wanda}-score ordering, or whether mixing both ends can help. For a linear layer \(l\), let \(N_l\) be the number of weights in the layer, \(s_l\) be the same rank-equivalent sparse budget used by \algname{Super}, and let \(A^{(l)}_{ij}=|W^{(l)}_{ij}|\|X^{(l)}_{j:}\|_2\) be the C4-calibrated \algname{Wanda}-style score. The hybrid mask is
\[
    M^{(l)}_{\mathrm{hyb}}(\beta)
    =
    \operatorname{TopK}\!\left(A^{(l)}, k_{\mathrm{top}}\right)
    \cup
    \operatorname{BottomK}\!\left(A^{(l)}, k_{\mathrm{bottom}}\right),
\]
where
\[
    k_{\mathrm{top}}=\operatorname{round}(\beta s_l),
    \qquad
    k_{\mathrm{bottom}}=s_l-k_{\mathrm{top}},
    \qquad
    \beta\in[0,1].
\]
Thus \(\beta\) is the fraction of the sparse budget assigned to the largest \algname{Wanda} scores. The endpoints recover the original sparse-only \algname{Super} variants: \(\beta=1\) is \algname{Super} (TopK), while \(\beta=0\) is \algname{Super} (BottomK). Intermediate values mix high-score and low-score trainable weights while preserving the same number of trainable sparse parameters.
Equivalently, if \(p_l=s_l/N_l\) is the realized layerwise sparse rate, then the split is \(k_{\mathrm{top}}\approx N_l p_l\beta\) and \(k_{\mathrm{bottom}}\approx N_l p_l(1-\beta)\).
The results below show that this mixed support can sometimes improve over pure TopK, but it does not improve the schedule-selected sparse-only results; pure BottomK remains the strongest mask in the main comparisons.

\begin{table}[htbp]
\centering
\scriptsize
\caption{Hybrid TopK--BottomK sparse-support ablation. Model: \texttt{Llama-3.2-1B}; fine-tuning dataset: Math17K; training budget: one or three epochs as indicated; rank-equivalent budget: \(r_0=8\). The Base row is frozen. All trainable masks use C4 calibration and approximately \(5.6\)M trainable sparse parameters. Best average within each epoch block is bold. Results are exact-answer accuracy (\%).}
\label{tab:llama1b_math17k_hybrid_super_accuracy_r8}
\resizebox{\textwidth}{!}{%
\begin{tabular}{lcccccccccc}
\toprule
\textbf{Method} & \textbf{Epochs} & \(\boldsymbol{\beta}\) & \textbf{Selected LR} & \textbf{AddSub} & \textbf{MultiArith} & \textbf{SingleEq} & \textbf{GSM8K} & \textbf{AQuA} & \textbf{SVAMP} & \textbf{Average} \\
\midrule
\textsc{Base} (frozen) & -- & -- & -- & 13.67 & 4.67 & 21.46 & 2.81 & 21.65 & 11.60 & 12.64 \\
\midrule
\algname{Super} (TopK) & 1 & 1.0 & $5\cdot10^{-4}$ & 47.59 & 85.50 & 57.87 & 24.26 & 16.93 & 37.10 & 44.88 \\
\algname{Super} (Hybrid) & 1 & 0.8 & $5\cdot10^{-4}$ & 42.78 & 85.50 & 59.84 & 17.13 & 0.39 & 37.30 & 40.49 \\
\algname{Super} (Hybrid) & 1 & 0.5 & $5\cdot10^{-4}$ & 40.51 & 86.83 & 58.27 & 28.20 & 10.24 & 37.90 & 43.66 \\
\algname{Super} (Hybrid) & 1 & 0.3 & $5\cdot10^{-4}$ & 60.76 & 87.17 & 58.27 & 27.22 & 22.44 & 46.20 & \textbf{50.34} \\
\algname{Super} (BottomK) & 1 & 0.0 & $10^{-3}$ & 39.75 & 90.17 & 34.65 & 16.30 & 15.35 & 25.80 & 37.00 \\
\midrule
\algname{Super} (TopK) & 3 & 1.0 & $10^{-4}$ & 31.90 & 87.00 & 22.44 & 7.43 & 22.44 & 7.30 & 29.75 \\
\algname{Super} (Hybrid) & 3 & 0.8 & $10^{-4}$ & 36.96 & 84.83 & 27.36 & 7.81 & 19.69 & 8.60 & 30.88 \\
\algname{Super} (Hybrid) & 3 & 0.5 & $10^{-4}$ & 58.73 & 85.83 & 57.48 & 20.92 & 20.47 & 33.10 & 46.09 \\
\algname{Super} (Hybrid) & 3 & 0.3 & $10^{-4}$ & 71.14 & 81.33 & 71.85 & 22.29 & 25.20 & 45.10 & 52.82 \\
\algname{Super} (BottomK) & 3 & 0.0 & $5\cdot10^{-4}$ & 69.11 & 85.50 & 70.47 & 34.50 & 25.59 & 47.60 & \textbf{55.46} \\
\bottomrule
\end{tabular}
}
\end{table}

\begin{table}[htbp]
\centering
\scriptsize
\caption{Hybrid TopK--BottomK sparse-support ablation. Model: \texttt{Meta-Llama-3-8B}; fine-tuning dataset: Math17K; training budget: three epochs; rank-equivalent budget: \(r_0=8\). The Base row is frozen. All trainable masks use C4 calibration and approximately \(21.0\)M trainable sparse parameters. Results are exact-answer accuracy (\%).}
\label{tab:llama8b_math17k_hybrid_super_accuracy_r8}
\resizebox{\textwidth}{!}{%
\begin{tabular}{lcccccccccc}
\toprule
\textbf{Method} & \textbf{Epochs} & \(\boldsymbol{\beta}\) & \textbf{Selected LR} & \textbf{AddSub} & \textbf{MultiArith} & \textbf{SingleEq} & \textbf{GSM8K} & \textbf{AQuA} & \textbf{SVAMP} & \textbf{Average} \\
\midrule
\textsc{Base} (frozen) & -- & -- & -- & 22.53 & 21.00 & 36.42 & 10.24 & 24.02 & 24.10 & 23.05 \\
\midrule
\algname{Super} (TopK) & 3 & 1.0 & $5\cdot10^{-5}$ & 4.81 & 88.17 & 13.19 & 9.48 & 35.04 & 5.90 & 26.10 \\
\algname{Super} (Hybrid) & 3 & 0.8 & $5\cdot10^{-5}$ & 4.05 & 97.00 & 21.65 & 15.16 & 40.16 & 11.50 & 31.59 \\
\algname{Super} (Hybrid) & 3 & 0.5 & $5\cdot10^{-5}$ & 0.51 & 98.17 & 21.85 & 36.24 & 36.61 & 13.90 & 34.55 \\
\algname{Super} (Hybrid) & 3 & 0.3 & $5\cdot10^{-5}$ & 15.19 & 98.67 & 60.43 & 51.25 & 42.52 & 43.20 & 51.88 \\
\algname{Super} (BottomK) & 3 & 0.0 & $10^{-4}$ & 72.15 & 95.67 & 88.58 & 66.34 & 42.91 & 70.70 & \textbf{72.73} \\
\bottomrule
\end{tabular}
}
\end{table}

\clearpage

\subsection{\texttt{Llama-3.2-1B} results}
\label{sec:llama1b_results_appendix}

\subsubsection{Math7K, one-epoch comparison}
\label{sec:llama1b_math7k_appendix}

\Cref{tab:llama1b_math7k_1epoch_accuracy_r8,tab:llama1b_math7k_1epoch_ppl_r8} report the \texttt{Llama-3.2-1B} Math7K one-epoch comparison. This run is kept in the appendix and is not used for the main-text claims. The \textbf{Calib.} column gives the calibration source used to construct fixed \algname{Wanda}-style sparse supports; ``--'' denotes methods that do not use a calibration pass. Full fine-tuning is included as an unbudgeted reference row separated from the matched-budget adapter rows.

\begin{table}[htbp]
\centering
\scriptsize
\caption{Arithmetic benchmark accuracy. Model: \texttt{Llama-3.2-1B}; fine-tuning dataset: Math7K; training budget: one epoch; rank-equivalent budget: \(r_0=8\). Trainable adapter rows use approximately \(5.6\)M trainable parameters; the Base row is frozen. Full fine-tuning is an unbudgeted reference row separated by rules. Bold marks the best observed matched-budget adapter average. Each row uses the learning rate selected by the held-out validation split of the fine-tuning set. Results are exact-answer accuracy (\%).}
\label{tab:llama1b_math7k_1epoch_accuracy_r8}
\resizebox{\textwidth}{!}{%
\begin{tabular}{llcccccccc}
\toprule
\textbf{Method} & \textbf{Calib.} & \textbf{Selected LR} & \textbf{AddSub} & \textbf{MultiArith} & \textbf{SingleEq} & \textbf{GSM8K} & \textbf{AQuA} & \textbf{SVAMP} & \textbf{Average} \\
\midrule
\textsc{Base} (frozen) & -- & -- & 13.67 & 4.67 & 21.46 & 2.81 & 21.65 & 11.60 & 12.64 \\
\midrule
Full fine-tuning & -- & $5\cdot10^{-5}$ & 36.71 & 67.50 & 49.80 & 25.78 & 21.65 & 35.20 & 39.44 \\
\midrule
\algname{LoRA} & -- & $10^{-3}$ & 36.96 & 64.67 & 50.00 & 21.23 & 21.26 & 34.20 & 38.05 \\
\algname{RoSA} & -- & $5\cdot10^{-4}$ & 39.75 & 65.50 & 51.57 & 22.90 & 22.44 & 31.80 & 38.99 \\
\algname{SIFT} (TopK) & -- & $10^{-4}$ & 29.87 & 55.00 & 38.39 & 16.22 & 20.47 & 22.70 & 30.44 \\
\algname{SIFT} (RandK) & -- & $10^{-3}$ & 36.20 & 67.33 & 50.59 & 20.85 & 19.69 & 29.70 & 37.39 \\
\algname{Super} (RandK) & -- & $10^{-3}$ & 38.73 & 62.83 & 48.03 & 21.68 & 20.47 & 30.60 & 37.06 \\
\algname{Super} (TopK) & C4 & $5\cdot10^{-4}$ & 35.95 & 63.83 & 48.62 & 21.46 & 21.65 & 29.00 & 36.75 \\
\algname{Super} (BottomK) & C4 & $10^{-3}$ & 40.76 & 59.33 & 52.56 & 21.76 & 17.32 & 32.30 & 37.34 \\
\algname{Supra} (TopK, $\lambda=0.3$) & C4 & $5\cdot10^{-4}$ & 38.73 & 62.50 & 53.35 & 21.30 & 23.23 & 32.40 & 38.59 \\
\algname{Supra} (TopK, $\lambda=0.5$) & C4 & $5\cdot10^{-4}$ & 38.48 & 64.67 & 51.77 & 23.12 & 21.26 & 31.80 & 38.52 \\
\algname{Supra} (TopK, $\lambda=0.8$) & C4 & $5\cdot10^{-4}$ & 41.01 & 65.83 & 52.76 & 22.44 & 21.65 & 34.30 & \textbf{39.67} \\
\algname{Supra} (BottomK, $\lambda=0.3$) & C4 & $10^{-3}$ & 34.18 & 62.50 & 46.46 & 23.96 & 21.26 & 33.00 & 36.89 \\
\algname{Supra} (BottomK, $\lambda=0.5$) & C4 & $10^{-3}$ & 38.23 & 62.67 & 49.21 & 23.05 & 24.02 & 33.20 & 38.40 \\
\algname{Supra} (BottomK, $\lambda=0.8$) & C4 & $10^{-3}$ & 33.16 & 64.83 & 51.38 & 22.52 & 22.83 & 34.00 & 38.12 \\
\algname{Super} (TopK) & Math7K & $5\cdot10^{-4}$ & 38.23 & 61.83 & 48.62 & 22.06 & 20.08 & 29.60 & 36.74 \\
\algname{Super} (BottomK) & Math7K & $10^{-3}$ & 41.01 & 57.83 & 51.38 & 20.17 & 20.08 & 32.20 & 37.11 \\
\bottomrule
\end{tabular}
}
\end{table}

\begin{table}[htbp]
\centering
\scriptsize
\caption{Test perplexity. Model: \texttt{Llama-3.2-1B}; fine-tuning dataset: Math7K; training budget: one epoch; rank-equivalent budget: \(r_0=8\). These are the same validation-selected checkpoints as \Cref{tab:llama1b_math7k_1epoch_accuracy_r8}. Full fine-tuning is an unbudgeted reference row separated by rules. Bold marks the lowest matched-budget adapter average PPL. The \textbf{Calib.} column gives the calibration source used to construct fixed \algname{Wanda}-style sparse supports; ``--'' denotes methods that do not use a calibration pass. Lower is better. Perplexities are computed on up to 120 examples per benchmark.}
\label{tab:llama1b_math7k_1epoch_ppl_r8}
\resizebox{\textwidth}{!}{%
\begin{tabular}{llcccccccc}
\toprule
\textbf{Method} & \textbf{Calib.} & \textbf{Selected LR} & \textbf{AddSub} & \textbf{MultiArith} & \textbf{SingleEq} & \textbf{GSM8K} & \textbf{AQuA} & \textbf{SVAMP} & \textbf{Average} \\
\midrule
\textsc{Base} (frozen) & -- & -- & 2.71 & 2.81 & 2.74 & 2.45 & 3.21 & 2.94 & 2.73 \\
\midrule
Full fine-tuning & -- & $5\cdot10^{-5}$ & 2.56 & 2.70 & 2.72 & 2.39 & 3.63 & 2.93 & 2.68 \\
\midrule
\algname{LoRA} & -- & $10^{-3}$ & 2.76 & 2.92 & 2.95 & 2.53 & 3.71 & 3.18 & 2.88 \\
\algname{RoSA} & -- & $5\cdot10^{-4}$ & 2.56 & 2.77 & 2.78 & 2.46 & 3.56 & 2.99 & 2.74 \\
\algname{SIFT} (TopK) & -- & $10^{-4}$ & 2.49 & 2.62 & 2.65 & 2.37 & 3.45 & 2.87 & 2.63 \\
\algname{SIFT} (RandK) & -- & $10^{-3}$ & 2.48 & 2.60 & 2.64 & 2.34 & 3.34 & 2.83 & 2.60 \\
\algname{Super} (RandK) & -- & $10^{-3}$ & 2.53 & 2.61 & 2.65 & 2.35 & 3.46 & 2.85 & 2.62 \\
\algname{Super} (TopK) & C4 & $5\cdot10^{-4}$ & 2.52 & 2.64 & 2.67 & 2.34 & 3.45 & 2.88 & 2.63 \\
\algname{Super} (BottomK) & C4 & $10^{-3}$ & 2.34 & 2.49 & 2.47 & 2.25 & 3.23 & 2.66 & 2.47 \\
\algname{Supra} (TopK, $\lambda=0.3$) & C4 & $5\cdot10^{-4}$ & 2.67 & 2.81 & 2.84 & 2.46 & 3.65 & 3.05 & 2.78 \\
\algname{Supra} (TopK, $\lambda=0.5$) & C4 & $5\cdot10^{-4}$ & 2.68 & 2.77 & 2.87 & 2.46 & 3.65 & 3.07 & 2.78 \\
\algname{Supra} (TopK, $\lambda=0.8$) & C4 & $5\cdot10^{-4}$ & 2.72 & 2.86 & 2.90 & 2.48 & 3.64 & 3.10 & 2.82 \\
\algname{Supra} (BottomK, $\lambda=0.3$) & C4 & $10^{-3}$ & 2.81 & 2.95 & 3.01 & 2.56 & 3.81 & 3.20 & 2.92 \\
\algname{Supra} (BottomK, $\lambda=0.5$) & C4 & $10^{-3}$ & 2.77 & 2.88 & 3.00 & 2.55 & 3.70 & 3.21 & 2.89 \\
\algname{Supra} (BottomK, $\lambda=0.8$) & C4 & $10^{-3}$ & 2.81 & 2.89 & 3.05 & 2.56 & 3.67 & 3.25 & 2.91 \\
\algname{Super} (TopK) & Math7K & $5\cdot10^{-4}$ & 2.49 & 2.63 & 2.62 & 2.34 & 3.49 & 2.81 & 2.60 \\
\algname{Super} (BottomK) & Math7K & $10^{-3}$ & 2.32 & 2.48 & 2.44 & 2.23 & 3.22 & 2.64 & \textbf{2.45} \\
\bottomrule
\end{tabular}
}
\end{table}

\clearpage

\subsubsection{Math17K, one-epoch comparison}
\label{sec:llama1b_math17k_appendix}
\label{sec:llama1b_math17k_1epoch_appendix}

\Cref{tab:llama1b_math17k_1epoch_accuracy_r8,tab:llama1b_math17k_1epoch_ppl_r8} report the \texttt{Llama-3.2-1B} Math17K comparison for one epoch. These rows are used together with the three-epoch rows to form the schedule-selected compact table in the main text. Full fine-tuning is included as an unbudgeted reference row separated from the matched-budget adapter rows. The separate hybrid TopK--BottomK sparse-support ablation is reported in \Cref{sec:hybrid_super_appendix}.

\begin{table}[htbp]
\centering
\scriptsize
\caption{Arithmetic benchmark accuracy. Model: \texttt{Llama-3.2-1B}; fine-tuning dataset: Math17K; training budget: one epoch; rank-equivalent budget: \(r_0=8\). The Base row is frozen. Adapter rows use approximately \(5.6\)M trainable parameters. Full fine-tuning is an unbudgeted reference row separated by rules. Bold marks the best observed matched-budget adapter average. Each row uses the learning rate selected by the held-out validation split of the fine-tuning set. Results are exact-answer accuracy (\%).}
\label{tab:llama1b_math17k_1epoch_accuracy_r8}
\resizebox{\textwidth}{!}{%
\begin{tabular}{llcccccccc}
\toprule
\textbf{Method} & \textbf{Calib.} & \textbf{Selected LR} & \textbf{AddSub} & \textbf{MultiArith} & \textbf{SingleEq} & \textbf{GSM8K} & \textbf{AQuA} & \textbf{SVAMP} & \textbf{Average} \\
\midrule
\textsc{Base} (frozen) & -- & -- & 13.67 & 4.67 & 21.46 & 2.81 & 21.65 & 11.60 & 12.64 \\
\midrule
Full fine-tuning & -- & $5\cdot10^{-5}$ & 70.38 & 91.83 & 77.76 & 36.54 & 24.02 & 58.10 & 59.77 \\
\midrule
\algname{LoRA} & -- & $10^{-3}$ & 73.67 & 85.83 & 77.56 & 30.10 & 24.41 & 53.80 & 57.56 \\
\algname{RoSA} & -- & $5\cdot10^{-4}$ & 71.65 & 89.17 & 78.35 & 27.98 & 24.80 & 53.70 & 57.61 \\
\algname{SIFT} (TopK) & -- & $10^{-4}$ & 60.51 & 83.17 & 63.19 & 23.43 & 3.94 & 40.30 & 45.75 \\
\algname{SIFT} (RandK) & -- & $10^{-3}$ & 62.28 & 90.83 & 73.62 & 31.46 & 17.72 & 48.70 & 54.10 \\
\algname{Super} (TopK) & C4 & $5\cdot10^{-4}$ & 47.59 & 85.50 & 57.87 & 24.26 & 16.93 & 37.10 & 44.88 \\
\algname{Super} (BottomK) & C4 & $10^{-3}$ & 39.75 & 90.17 & 34.65 & 16.30 & 15.35 & 25.80 & 37.00 \\
\algname{Magnitude} (TopK) & -- & $10^{-3}$ & 77.47 & 1.50 & 80.71 & 27.75 & 20.87 & 52.10 & 43.40 \\
\algname{Magnitude} (BottomK) & -- & $10^{-3}$ & 73.92 & 90.50 & 81.69 & 30.25 & 24.80 & 54.70 & 59.31 \\
\algname{Supra} (TopK, $\lambda=0.3$) & C4 & $5\cdot10^{-4}$ & 69.87 & 90.33 & 78.74 & 26.69 & 24.80 & 52.00 & 57.07 \\
\algname{Supra} (TopK, $\lambda=0.5$) & C4 & $5\cdot10^{-4}$ & 61.77 & 86.33 & 54.53 & 24.72 & 25.98 & 35.30 & 48.11 \\
\algname{Supra} (TopK, $\lambda=0.8$) & C4 & $5\cdot10^{-4}$ & 78.23 & 92.17 & 80.71 & 23.65 & 25.98 & 57.40 & \textbf{59.69} \\
\algname{Supra} (BottomK, $\lambda=0.3$) & C4 & $10^{-3}$ & 59.24 & 87.33 & 57.48 & 23.12 & 2.36 & 31.50 & 43.51 \\
\algname{Supra} (BottomK, $\lambda=0.5$) & C4 & $10^{-3}$ & 76.20 & 79.67 & 79.92 & 28.35 & 24.02 & 59.90 & 58.01 \\
\algname{Supra} (BottomK, $\lambda=0.8$) & C4 & $10^{-3}$ & 74.18 & 89.33 & 79.53 & 29.95 & 26.77 & 53.00 & 58.79 \\
\algname{Supra-Mag} (BottomK, $\lambda=0.3$) & -- & $5\cdot10^{-4}$ & 76.71 & 91.67 & 81.10 & 22.97 & 27.56 & 55.10 & 59.18 \\
\algname{Supra-Mag} (BottomK, $\lambda=0.5$) & -- & $10^{-3}$ & 66.84 & 88.17 & 78.15 & 30.02 & 22.05 & 49.50 & 55.79 \\
\algname{Supra-Mag} (BottomK, $\lambda=0.8$) & -- & $10^{-3}$ & 73.16 & 88.83 & 75.79 & 33.13 & 24.02 & 52.00 & 57.82 \\
\bottomrule
\end{tabular}
}
\end{table}

\begin{table}[htbp]
\centering
\scriptsize
\caption{Test perplexity. Model: \texttt{Llama-3.2-1B}; fine-tuning dataset: Math17K; training budget: one epoch; rank-equivalent budget: \(r_0=8\). These are the same validation-selected checkpoints as \Cref{tab:llama1b_math17k_1epoch_accuracy_r8}. The Base row is frozen. Full fine-tuning is an unbudgeted reference row separated by rules. Bold marks the lowest matched-budget adapter average PPL. The \textbf{Calib.} column gives the calibration source used to construct fixed \algname{Wanda}-style sparse supports; ``--'' denotes methods that do not use a calibration pass. Lower is better. Perplexities are computed on up to 120 examples per benchmark.}
\label{tab:llama1b_math17k_1epoch_ppl_r8}
\resizebox{\textwidth}{!}{%
\begin{tabular}{llcccccccc}
\toprule
\textbf{Method} & \textbf{Calib.} & \textbf{Selected LR} & \textbf{AddSub} & \textbf{MultiArith} & \textbf{SingleEq} & \textbf{GSM8K} & \textbf{AQuA} & \textbf{SVAMP} & \textbf{Average} \\
\midrule
\textsc{Base} (frozen) & -- & -- & 2.71 & 2.81 & 2.74 & 2.45 & 3.21 & 2.94 & 2.73 \\
\midrule
Full fine-tuning & -- & $5\cdot10^{-5}$ & 1.10 & 1.12 & 1.14 & 1.28 & 1.59 & 1.19 & 1.21 \\
\midrule
\algname{LoRA} & -- & $10^{-3}$ & 1.12 & 1.14 & 1.16 & 1.32 & 1.66 & 1.22 & 1.24 \\
\algname{RoSA} & -- & $5\cdot10^{-4}$ & 1.13 & 1.16 & 1.18 & 1.35 & 1.73 & 1.24 & 1.27 \\
\algname{SIFT} (TopK) & -- & $10^{-4}$ & 1.18 & 1.23 & 1.23 & 1.40 & 1.83 & 1.31 & 1.33 \\
\algname{SIFT} (RandK) & -- & $10^{-3}$ & 1.12 & 1.15 & 1.17 & 1.32 & 1.68 & 1.23 & 1.25 \\
\algname{Super} (TopK) & C4 & $5\cdot10^{-4}$ & 1.13 & 1.17 & 1.18 & 1.34 & 1.73 & 1.25 & 1.27 \\
\algname{Super} (BottomK) & C4 & $10^{-3}$ & 1.12 & 1.15 & 1.17 & 1.32 & 1.80 & 1.23 & 1.26 \\
\algname{Magnitude} (TopK) & -- & $10^{-3}$ & 1.12 & 1.16 & 1.17 & 1.33 & 1.73 & 1.23 & 1.26 \\
\algname{Magnitude} (BottomK) & -- & $10^{-3}$ & 1.12 & 1.14 & 1.16 & 1.31 & 1.64 & 1.22 & 1.24 \\
\algname{Supra} (TopK, $\lambda=0.3$) & C4 & $5\cdot10^{-4}$ & 1.13 & 1.15 & 1.17 & 1.33 & 1.71 & 1.23 & 1.26 \\
\algname{Supra} (TopK, $\lambda=0.5$) & C4 & $5\cdot10^{-4}$ & 1.12 & 1.15 & 1.17 & 1.33 & 1.69 & 1.23 & 1.25 \\
\algname{Supra} (TopK, $\lambda=0.8$) & C4 & $5\cdot10^{-4}$ & 1.13 & 1.16 & 1.17 & 1.34 & 1.70 & 1.24 & 1.26 \\
\algname{Supra} (BottomK, $\lambda=0.3$) & C4 & $10^{-3}$ & 1.12 & 1.14 & 1.16 & 1.32 & 1.65 & 1.22 & 1.24 \\
\algname{Supra} (BottomK, $\lambda=0.5$) & C4 & $10^{-3}$ & 1.13 & 1.13 & 1.16 & 1.31 & 1.64 & 1.22 & 1.24 \\
\algname{Supra} (BottomK, $\lambda=0.8$) & C4 & $10^{-3}$ & 1.11 & 1.13 & 1.16 & 1.31 & 1.65 & 1.21 & \textbf{1.24} \\
\algname{Supra-Mag} (BottomK, $\lambda=0.3$) & -- & $5\cdot10^{-4}$ & 1.12 & 1.15 & 1.17 & 1.32 & 1.67 & 1.23 & 1.25 \\
\algname{Supra-Mag} (BottomK, $\lambda=0.5$) & -- & $10^{-3}$ & 1.12 & 1.13 & 1.16 & 1.32 & 1.65 & 1.22 & 1.24 \\
\algname{Supra-Mag} (BottomK, $\lambda=0.8$) & -- & $10^{-3}$ & 1.12 & 1.13 & 1.15 & 1.31 & 1.65 & 1.21 & 1.24 \\
\bottomrule
\end{tabular}
}
\end{table}

\clearpage

\subsubsection{Math17K, three-epoch comparison}
\label{sec:llama1b_math17k_3epoch_appendix}

\Cref{tab:llama1b_math17k_3epoch_accuracy_r8,tab:llama1b_math17k_3epoch_ppl_r8} report the \texttt{Llama-3.2-1B} Math17K comparison for three epochs. These runs use the same rank-equivalent adapter budget and evaluation suite as the one-epoch comparison, with full fine-tuning included as an unbudgeted reference row.

\begin{figure}[htbp]
\centering
\begin{minipage}{0.49\textwidth}
\centering
\includegraphics[width=\linewidth]{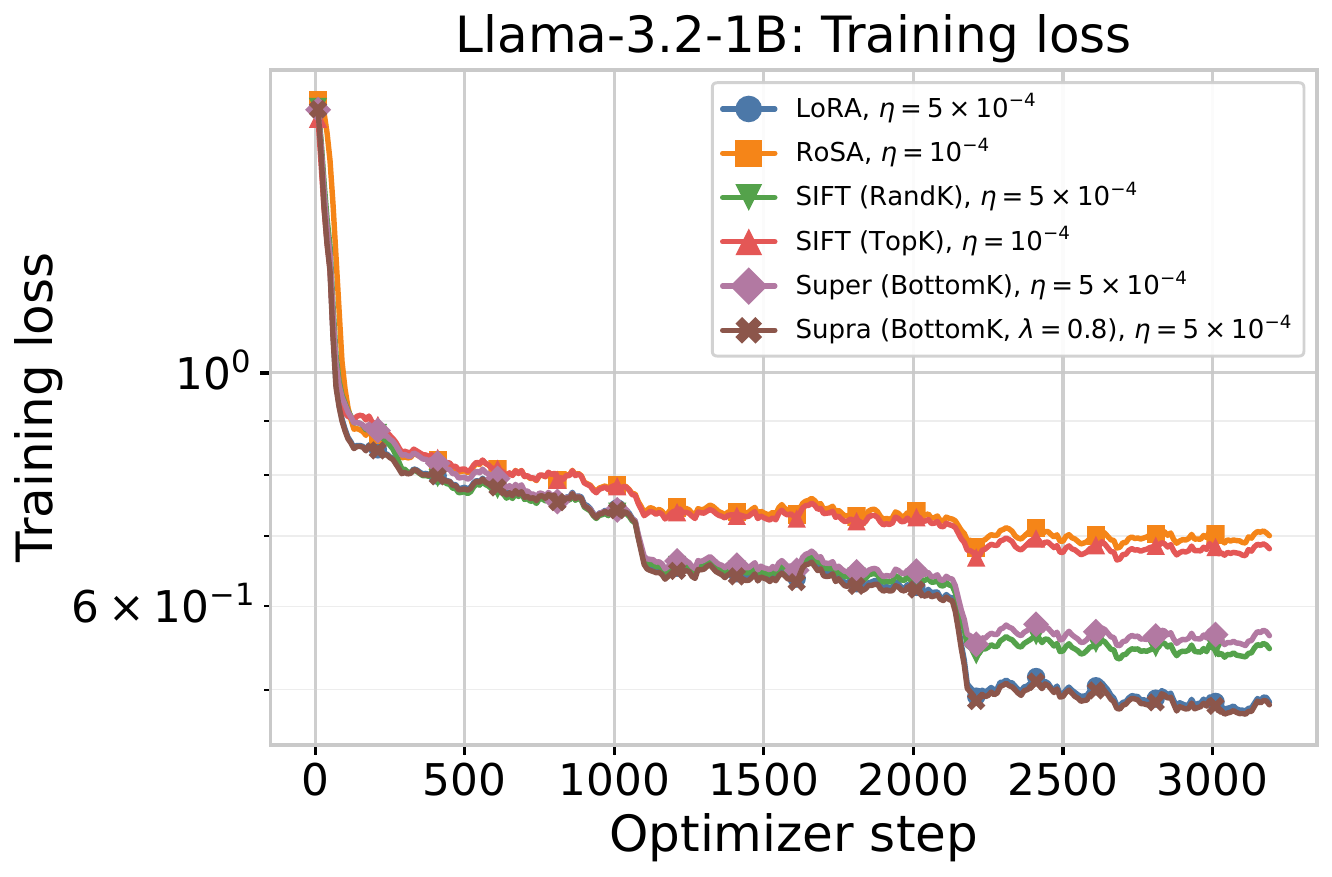}
\end{minipage}
\hfill
\begin{minipage}{0.49\textwidth}
\centering
\includegraphics[width=\linewidth]{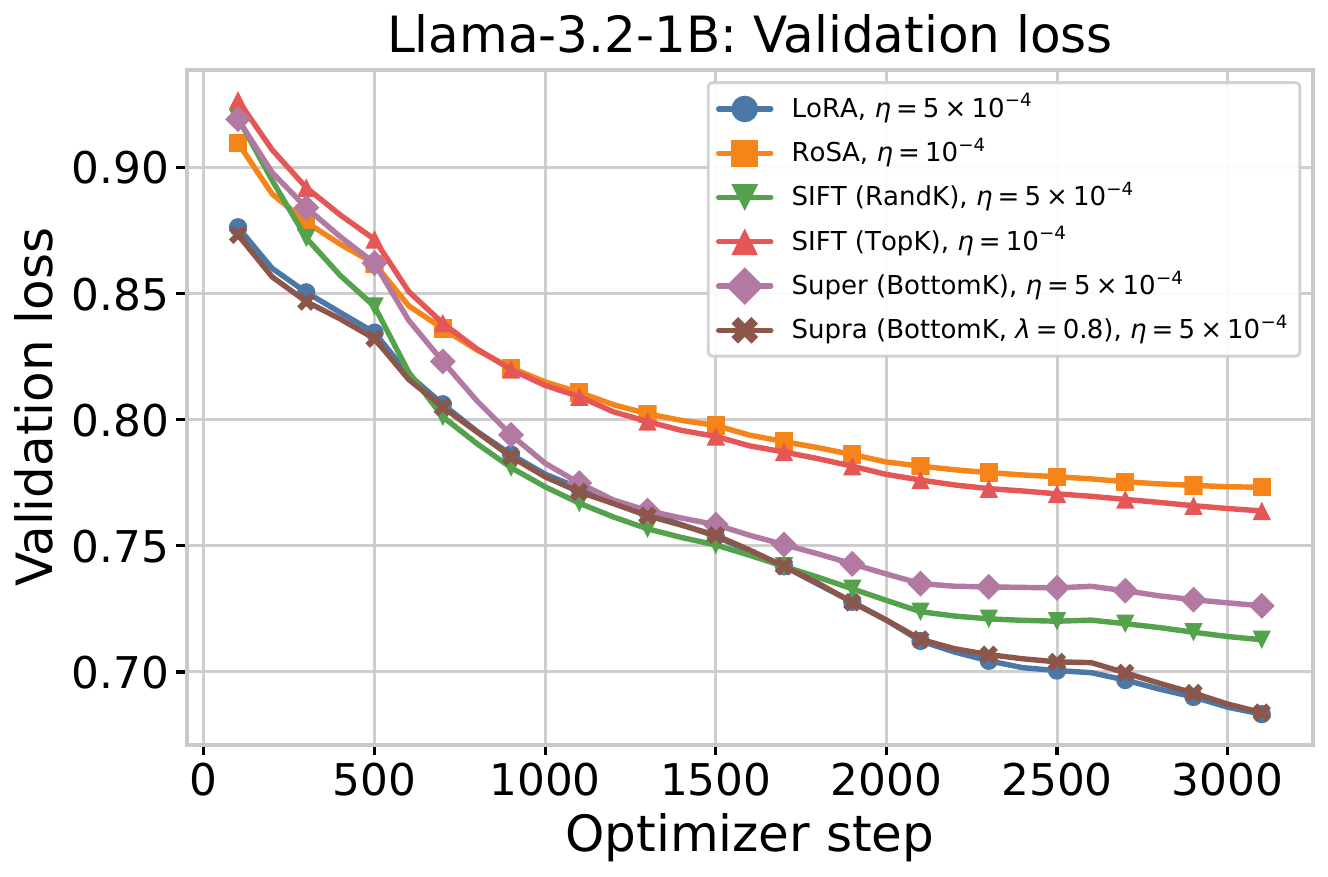}
\end{minipage}
\caption{Optimization curves for \texttt{Llama-3.2-1B} Math17K fine-tuning over three epochs. Left: training loss, plotted with a log-scaled \(y\)-axis. Right: held-out Math17K validation loss. Each curve uses the validation-selected learning rate for the corresponding method under the rank-equivalent \(r_0=8\) budget.}
\label{fig:llama1b_math17k_optimization_curves_appendix}
\end{figure}

\begin{table}[htbp]
\centering
\scriptsize
\caption{Arithmetic benchmark accuracy. Model: \texttt{Llama-3.2-1B}; fine-tuning dataset: Math17K; training budget: three epochs; rank-equivalent budget: \(r_0=8\). The Base row is frozen. Adapter rows use approximately \(5.6\)M trainable parameters. Full fine-tuning is an unbudgeted reference row separated by rules. Bold marks the best observed matched-budget adapter average. Each row uses the learning rate selected by the held-out validation split of the fine-tuning set. Results are exact-answer accuracy (\%).}
\label{tab:llama1b_math17k_3epoch_accuracy_r8}
\resizebox{\textwidth}{!}{%
\begin{tabular}{llcccccccc}
\toprule
\textbf{Method} & \textbf{Calib.} & \textbf{Selected LR} & \textbf{AddSub} & \textbf{MultiArith} & \textbf{SingleEq} & \textbf{GSM8K} & \textbf{AQuA} & \textbf{SVAMP} & \textbf{Average} \\
\midrule
\textsc{Base} (frozen) & -- & -- & 13.67 & 4.67 & 21.46 & 2.81 & 21.65 & 11.60 & 12.64 \\
\midrule
Full fine-tuning & -- & $5\cdot10^{-5}$ & 80.00 & 91.67 & 85.63 & 44.05 & 28.74 & 64.10 & 65.70 \\
\midrule
\algname{LoRA} & -- & $5\cdot10^{-4}$ & 70.63 & 90.83 & 85.24 & 35.33 & 26.77 & 57.60 & 61.07 \\
\algname{RoSA} & -- & $10^{-4}$ & 59.75 & 89.00 & 71.26 & 27.82 & 21.26 & 46.60 & 52.62 \\
\algname{SIFT} (TopK) & -- & $10^{-4}$ & 30.89 & 64.17 & 16.73 & 7.51 & 24.80 & 7.40 & 25.25 \\
\algname{SIFT} (RandK) & -- & $5\cdot10^{-4}$ & 52.15 & 81.67 & 72.83 & 29.72 & 25.98 & 53.70 & 52.68 \\
\algname{Super} (TopK) & C4 & $10^{-4}$ & 31.90 & 87.00 & 22.44 & 7.43 & 22.44 & 7.30 & 29.75 \\
\algname{Super} (BottomK) & C4 & $5\cdot10^{-4}$ & 69.11 & 85.50 & 70.47 & 34.50 & 25.59 & 47.60 & 55.46 \\
\algname{Magnitude} (TopK) & -- & $5\cdot10^{-4}$ & 72.91 & 84.83 & 73.03 & 33.43 & 23.23 & 45.60 & 55.51 \\
\algname{Magnitude} (BottomK) & -- & $5\cdot10^{-4}$ & 61.77 & 25.00 & 69.09 & 33.81 & 25.20 & 46.70 & 43.60 \\
\algname{Supra} (TopK, $\lambda=0.3$) & C4 & $10^{-4}$ & 45.06 & 90.83 & 72.05 & 32.60 & 24.80 & 47.70 & 52.17 \\
\algname{Supra} (TopK, $\lambda=0.5$) & C4 & $10^{-4}$ & 63.80 & 81.17 & 75.59 & 31.69 & 23.62 & 52.20 & 54.68 \\
\algname{Supra} (TopK, $\lambda=0.8$) & C4 & $10^{-4}$ & 64.56 & 76.00 & 72.05 & 30.10 & 24.02 & 47.00 & 52.29 \\
\algname{Supra} (BottomK, $\lambda=0.3$) & C4 & $10^{-4}$ & 60.25 & 89.83 & 72.05 & 31.77 & 25.59 & 50.70 & 55.03 \\
\algname{Supra} (BottomK, $\lambda=0.5$) & C4 & $10^{-4}$ & 61.27 & 92.17 & 74.02 & 30.25 & 20.47 & 54.20 & 55.40 \\
\algname{Supra} (BottomK, $\lambda=0.8$) & C4 & $5\cdot10^{-4}$ & 82.53 & 85.50 & 84.45 & 36.16 & 24.41 & 60.30 & \textbf{62.23} \\
\algname{Supra-Mag} (BottomK, $\lambda=0.3$) & -- & $10^{-4}$ & 60.00 & 89.33 & 68.70 & 33.74 & 25.98 & 44.50 & 53.71 \\
\algname{Supra-Mag} (BottomK, $\lambda=0.5$) & -- & $10^{-4}$ & 60.25 & 90.50 & 75.00 & 31.24 & 25.98 & 54.00 & 56.16 \\
\algname{Supra-Mag} (BottomK, $\lambda=0.8$) & -- & $5\cdot10^{-4}$ & 76.46 & 28.17 & 80.51 & 38.29 & 25.59 & 52.60 & 50.27 \\
\bottomrule
\end{tabular}
}
\end{table}

\begin{table}[htbp]
\centering
\scriptsize
\caption{Test perplexity. Model: \texttt{Llama-3.2-1B}; fine-tuning dataset: Math17K; training budget: three epochs; rank-equivalent budget: \(r_0=8\). These are the same validation-selected checkpoints as \Cref{tab:llama1b_math17k_3epoch_accuracy_r8}. The Base row is frozen. Full fine-tuning is an unbudgeted reference row separated by rules. Bold marks the lowest matched-budget adapter average PPL. The \textbf{Calib.} column gives the calibration source used to construct fixed \algname{Wanda}-style sparse supports; ``--'' denotes methods that do not use a calibration pass. Lower is better. Perplexities are computed on up to 120 examples per benchmark.}
\label{tab:llama1b_math17k_3epoch_ppl_r8}
\resizebox{\textwidth}{!}{%
\begin{tabular}{llcccccccc}
\toprule
\textbf{Method} & \textbf{Calib.} & \textbf{Selected LR} & \textbf{AddSub} & \textbf{MultiArith} & \textbf{SingleEq} & \textbf{GSM8K} & \textbf{AQuA} & \textbf{SVAMP} & \textbf{Average} \\
\midrule
\textsc{Base} (frozen) & -- & -- & 2.71 & 2.81 & 2.74 & 2.45 & 3.21 & 2.94 & 2.73 \\
\midrule
Full fine-tuning & -- & $5\cdot10^{-5}$ & 1.05 & 1.05 & 1.07 & 1.16 & 1.29 & 1.10 & 1.11 \\
\midrule
\algname{LoRA} & -- & $5\cdot10^{-4}$ & 1.07 & 1.07 & 1.11 & 1.23 & 1.45 & 1.15 & 1.17 \\
\algname{RoSA} & -- & $10^{-4}$ & 1.13 & 1.16 & 1.18 & 1.34 & 1.70 & 1.24 & 1.27 \\
\algname{SIFT} (TopK) & -- & $10^{-4}$ & 1.12 & 1.16 & 1.17 & 1.33 & 1.68 & 1.24 & 1.26 \\
\algname{SIFT} (RandK) & -- & $5\cdot10^{-4}$ & 1.08 & 1.10 & 1.12 & 1.26 & 1.56 & 1.18 & 1.19 \\
\algname{Super} (TopK) & C4 & $10^{-4}$ & 1.16 & 1.18 & 1.21 & 1.39 & 1.97 & 1.29 & 1.31 \\
\algname{Super} (BottomK) & C4 & $5\cdot10^{-4}$ & 1.09 & 1.10 & 1.13 & 1.27 & 1.72 & 1.18 & 1.21 \\
\algname{Magnitude} (TopK) & -- & $5\cdot10^{-4}$ & 1.09 & 1.10 & 1.13 & 1.27 & 1.72 & 1.18 & 1.21 \\
\algname{Magnitude} (BottomK) & -- & $5\cdot10^{-4}$ & 1.08 & 1.09 & 1.12 & 1.25 & 1.59 & 1.17 & 1.19 \\
\algname{Supra} (TopK, $\lambda=0.3$) & C4 & $10^{-4}$ & 1.11 & 1.13 & 1.15 & 1.31 & 1.76 & 1.21 & 1.24 \\
\algname{Supra} (TopK, $\lambda=0.5$) & C4 & $10^{-4}$ & 1.11 & 1.13 & 1.15 & 1.31 & 1.67 & 1.21 & 1.24 \\
\algname{Supra} (TopK, $\lambda=0.8$) & C4 & $10^{-4}$ & 1.11 & 1.13 & 1.16 & 1.32 & 1.71 & 1.22 & 1.24 \\
\algname{Supra} (BottomK, $\lambda=0.3$) & C4 & $10^{-4}$ & 1.13 & 1.14 & 1.18 & 1.34 & 1.71 & 1.24 & 1.26 \\
\algname{Supra} (BottomK, $\lambda=0.5$) & C4 & $10^{-4}$ & 1.12 & 1.15 & 1.17 & 1.34 & 1.81 & 1.24 & 1.26 \\
\algname{Supra} (BottomK, $\lambda=0.8$) & C4 & $5\cdot10^{-4}$ & 1.08 & 1.08 & 1.10 & 1.23 & 1.44 & 1.14 & \textbf{1.16} \\
\algname{Supra-Mag} (BottomK, $\lambda=0.3$) & -- & $10^{-4}$ & 1.14 & 1.15 & 1.18 & 1.34 & 1.79 & 1.25 & 1.27 \\
\algname{Supra-Mag} (BottomK, $\lambda=0.5$) & -- & $10^{-4}$ & 1.13 & 1.15 & 1.18 & 1.34 & 1.79 & 1.24 & 1.27 \\
\algname{Supra-Mag} (BottomK, $\lambda=0.8$) & -- & $5\cdot10^{-4}$ & 1.07 & 1.08 & 1.10 & 1.22 & 1.44 & 1.14 & 1.16 \\
\bottomrule
\end{tabular}
}
\end{table}

\subsection{\texttt{Meta-Llama-3-8B} results}
\label{sec:llama8b_results_appendix}

\subsubsection{Math7K, one-epoch comparison}
\label{sec:llama8b_math7k_appendix}

\Cref{tab:llama8b_math7k_1epoch_accuracy_r8,tab:llama8b_math7k_1epoch_ppl_r8} report the Math7K one-epoch comparison for \texttt{Meta-Llama-3-8B}. We only report the completed 8B Math7K adapter suite here; the larger Math17K 8B comparisons are reported separately in \Cref{sec:llama8b_math17k_1epoch_appendix,sec:llama8b_math17k_appendix}.

\begin{table}[htbp]
\centering
\scriptsize
\caption{Arithmetic benchmark accuracy. Model: \texttt{Meta-Llama-3-8B}; fine-tuning dataset: Math7K; training budget: one epoch; rank-equivalent budget: \(r_0=8\). Adapter rows use approximately the rank-8 \algname{LoRA} trainable-parameter budget for this model. Each row uses the learning rate selected by the held-out validation split of the fine-tuning set. Results are exact-answer accuracy (\%).}
\label{tab:llama8b_math7k_1epoch_accuracy_r8}
\resizebox{\textwidth}{!}{%
\begin{tabular}{llcccccccc}
\toprule
\textbf{Method} & \textbf{Calib.} & \textbf{Selected LR} & \textbf{AddSub} & \textbf{MultiArith} & \textbf{SingleEq} & \textbf{GSM8K} & \textbf{AQuA} & \textbf{SVAMP} & \textbf{Average} \\
\midrule
\textsc{Base} (frozen) & -- & -- & 22.53 & 21.00 & 36.42 & 10.24 & 24.02 & 24.10 & 23.05 \\
\algname{LoRA} & -- & $5\cdot10^{-4}$ & 83.29 & 91.17 & 92.72 & 65.35 & 25.98 & 75.80 & 72.39 \\
\algname{RoSA} & -- & $10^{-4}$ & 82.78 & 87.50 & 91.34 & 64.37 & 28.74 & 74.70 & 71.57 \\
\algname{SIFT} (TopK) & -- & $10^{-4}$ & 82.28 & 89.17 & 90.94 & 62.40 & 25.20 & 72.60 & 70.43 \\
\algname{SIFT} (RandK) & -- & $10^{-3}$ & 78.73 & 92.50 & 91.14 & 65.20 & 25.59 & 74.70 & 71.31 \\
\algname{Super} (RandK) & -- & $10^{-3}$ & 81.77 & 90.83 & 92.32 & 65.88 & 27.17 & 74.80 & 72.13 \\
\algname{Super} (TopK) & C4 & $10^{-4}$ & 84.81 & 90.00 & 88.78 & 62.17 & 26.38 & 75.80 & 71.32 \\
\algname{Super} (BottomK) & C4 & $10^{-3}$ & 81.77 & 89.67 & 92.13 & 65.50 & 27.56 & 74.00 & 71.77 \\
\algname{Supra} (TopK, $\lambda=0.3$) & C4 & $10^{-4}$ & 84.30 & 89.00 & 91.14 & 63.23 & 27.17 & 74.50 & 71.56 \\
\algname{Supra} (TopK, $\lambda=0.5$) & C4 & $10^{-4}$ & 83.80 & 89.83 & 90.75 & 62.70 & 30.31 & 73.40 & 71.80 \\
\algname{Supra} (TopK, $\lambda=0.8$) & C4 & $10^{-4}$ & 85.32 & 88.83 & 92.32 & 63.68 & 27.56 & 74.90 & 72.10 \\
\algname{Supra} (BottomK, $\lambda=0.3$) & C4 & $5\cdot10^{-4}$ & 81.77 & 91.17 & 91.54 & 66.94 & 28.74 & 75.50 & 72.61 \\
\algname{Supra} (BottomK, $\lambda=0.5$) & C4 & $5\cdot10^{-4}$ & 81.01 & 89.67 & 90.94 & 66.19 & 26.77 & 77.60 & 72.03 \\
\algname{Supra} (BottomK, $\lambda=0.8$) & C4 & $5\cdot10^{-4}$ & 83.04 & 91.83 & 91.34 & 66.03 & 26.77 & 77.20 & \textbf{72.70} \\
\bottomrule
\end{tabular}
}
\end{table}

\begin{table}[htbp]
\centering
\scriptsize
\caption{Test perplexity. Model: \texttt{Meta-Llama-3-8B}; fine-tuning dataset: Math7K; training budget: one epoch; rank-equivalent budget: \(r_0=8\). These are the same validation-selected checkpoints as \Cref{tab:llama8b_math7k_1epoch_accuracy_r8}. The \textbf{Calib.} column gives the calibration source used to construct fixed \algname{Wanda}-style sparse supports; ``--'' denotes methods that do not use a calibration pass. Lower is better. Perplexities are computed on up to 120 examples per benchmark.}
\label{tab:llama8b_math7k_1epoch_ppl_r8}
\resizebox{\textwidth}{!}{%
\begin{tabular}{llcccccccc}
\toprule
\textbf{Method} & \textbf{Calib.} & \textbf{Selected LR} & \textbf{AddSub} & \textbf{MultiArith} & \textbf{SingleEq} & \textbf{GSM8K} & \textbf{AQuA} & \textbf{SVAMP} & \textbf{Average} \\
\midrule
\textsc{Base} (frozen) & -- & -- & 2.15 & 2.28 & 2.30 & 1.98 & 2.64 & 2.41 & \textbf{2.22} \\
\algname{LoRA} & -- & $5\cdot10^{-4}$ & 2.34 & 2.50 & 2.53 & 2.15 & 2.97 & 2.64 & 2.43 \\
\algname{RoSA} & -- & $10^{-4}$ & 2.37 & 2.54 & 2.61 & 2.20 & 3.05 & 2.73 & 2.49 \\
\algname{SIFT} (TopK) & -- & $10^{-4}$ & 2.43 & 2.58 & 2.74 & 2.25 & 2.95 & 2.84 & 2.55 \\
\algname{SIFT} (RandK) & -- & $10^{-3}$ & 2.20 & 2.37 & 2.38 & 2.09 & 2.86 & 2.49 & 2.32 \\
\algname{Super} (RandK) & -- & $10^{-3}$ & 2.16 & 2.30 & 2.32 & 2.05 & 2.86 & 2.44 & 2.27 \\
\algname{Super} (TopK) & C4 & $10^{-4}$ & 2.21 & 2.38 & 2.44 & 2.08 & 2.86 & 2.55 & 2.33 \\
\algname{Super} (BottomK) & C4 & $10^{-3}$ & 2.38 & 2.51 & 2.58 & 2.19 & 3.00 & 2.69 & 2.47 \\
\algname{Supra} (TopK, $\lambda=0.3$) & C4 & $10^{-4}$ & 2.26 & 2.43 & 2.45 & 2.11 & 2.93 & 2.57 & 2.37 \\
\algname{Supra} (TopK, $\lambda=0.5$) & C4 & $10^{-4}$ & 2.28 & 2.45 & 2.49 & 2.12 & 2.91 & 2.60 & 2.39 \\
\algname{Supra} (TopK, $\lambda=0.8$) & C4 & $10^{-4}$ & 2.36 & 2.55 & 2.57 & 2.17 & 2.92 & 2.69 & 2.46 \\
\algname{Supra} (BottomK, $\lambda=0.3$) & C4 & $5\cdot10^{-4}$ & 2.18 & 2.37 & 2.35 & 2.07 & 2.91 & 2.49 & 2.31 \\
\algname{Supra} (BottomK, $\lambda=0.5$) & C4 & $5\cdot10^{-4}$ & 2.18 & 2.37 & 2.37 & 2.08 & 2.89 & 2.50 & 2.31 \\
\algname{Supra} (BottomK, $\lambda=0.8$) & C4 & $5\cdot10^{-4}$ & 2.26 & 2.48 & 2.43 & 2.11 & 2.89 & 2.55 & 2.37 \\
\bottomrule
\end{tabular}
}
\end{table}

\clearpage

\subsubsection{Math17K, one-epoch comparison}
\label{sec:llama8b_math17k_1epoch_appendix}

\Cref{tab:llama8b_math17k_1epoch_accuracy_r8,tab:llama8b_math17k_1epoch_ppl_r8} report the \texttt{Meta-Llama-3-8B} Math17K comparison for one epoch. The \algname{RoSA} row uses the corrected one-epoch \algname{RoSA} run; the remaining adapter rows come from the matching one-epoch sweep without hybrid TopK--BottomK sparse supports. Full fine-tuning is included as an unbudgeted reference row. The separate fixed-learning-rate sanity check for \algname{Super} (BottomK) is exported with the result artifacts, but is not included in the table because it was not selected by the validation-loss protocol.

\begin{table}[htbp]
\centering
\scriptsize
\caption{Arithmetic benchmark accuracy. Model: \texttt{Meta-Llama-3-8B}; fine-tuning dataset: Math17K; training budget: one epoch; rank-equivalent budget: \(r_0=8\). Adapter rows use approximately the rank-8 \algname{LoRA} trainable-parameter budget for this model. Full fine-tuning is an unbudgeted reference row separated by rules. Bold marks the best observed matched-budget adapter average. The \textbf{Calib.} column gives the calibration source used to construct fixed \algname{Wanda}-style sparse supports; ``--'' denotes methods that do not use a calibration pass. Each row uses the learning rate selected by the held-out validation split of the fine-tuning set. Results are exact-answer accuracy (\%).}
\label{tab:llama8b_math17k_1epoch_accuracy_r8}
\resizebox{\textwidth}{!}{%
\begin{tabular}{llcccccccc}
\toprule
\textbf{Method} & \textbf{Calib.} & \textbf{Selected LR} & \textbf{AddSub} & \textbf{MultiArith} & \textbf{SingleEq} & \textbf{GSM8K} & \textbf{AQuA} & \textbf{SVAMP} & \textbf{Average} \\
\midrule
\textsc{Base} (frozen) & -- & -- & 22.53 & 21.00 & 36.42 & 10.24 & 24.02 & 24.10 & 23.05 \\
\midrule
Full fine-tuning & -- & $5\cdot10^{-5}$ & 86.08 & 47.50 & 90.35 & 50.64 & 28.35 & 76.00 & 63.15 \\
\midrule
\algname{LoRA} & -- & $5\cdot10^{-4}$ & 86.33 & 97.67 & 93.70 & 66.34 & 24.41 & 70.60 & 73.17 \\
\algname{RoSA} & -- & $10^{-4}$ & 86.84 & 97.50 & 92.72 & 58.45 & 35.83 & 79.60 & 75.16 \\
\algname{SIFT} (TopK) & -- & $10^{-4}$ & 86.58 & 98.50 & 94.69 & 63.76 & 35.04 & 79.10 & 76.28 \\
\algname{SIFT} (RandK) & -- & $10^{-3}$ & 84.30 & 98.33 & 93.50 & 63.99 & 45.28 & 81.60 & 77.83 \\
\algname{Super} (RandK) & -- & $10^{-3}$ & 6.58 & 99.17 & 16.93 & 34.95 & 35.83 & 9.00 & 33.74 \\
\algname{Super} (TopK) & C4 & $5\cdot10^{-4}$ & 13.67 & 90.67 & 24.41 & 45.19 & 16.54 & 16.00 & 34.41 \\
\algname{Super} (BottomK) & C4 & $10^{-3}$ & 1.27 & 98.17 & 3.94 & 2.58 & 27.17 & 0.60 & 22.29 \\
\algname{Magnitude} (TopK) & -- & $10^{-3}$ & 54.68 & 98.00 & 61.42 & 57.85 & 37.40 & 35.10 & 57.41 \\
\algname{Magnitude} (BottomK) & -- & $10^{-3}$ & 90.89 & 97.33 & 95.08 & 68.01 & 41.73 & 81.10 & 79.02 \\
\algname{Supra} (TopK, $\lambda=0.3$) & C4 & $10^{-4}$ & 86.84 & 86.17 & 87.40 & 71.49 & 35.83 & 82.00 & 74.95 \\
\algname{Supra} (TopK, $\lambda=0.5$) & C4 & $10^{-4}$ & 83.04 & 98.00 & 86.02 & 70.20 & 37.80 & 80.20 & 75.88 \\
\algname{Supra} (TopK, $\lambda=0.8$) & C4 & $10^{-4}$ & 88.10 & 99.00 & 96.65 & 69.90 & 30.71 & 81.70 & 77.68 \\
\algname{Supra} (BottomK, $\lambda=0.3$) & C4 & $5\cdot10^{-4}$ & 89.11 & 95.33 & 96.06 & 67.63 & 41.34 & 82.50 & 78.66 \\
\algname{Supra} (BottomK, $\lambda=0.5$) & C4 & $5\cdot10^{-4}$ & 91.39 & 97.50 & 94.09 & 67.85 & 33.86 & 83.40 & 78.02 \\
\algname{Supra} (BottomK, $\lambda=0.8$) & C4 & $5\cdot10^{-4}$ & 90.89 & 96.67 & 96.26 & 67.17 & 36.22 & 82.50 & 78.28 \\
\algname{Supra-Mag} (BottomK, $\lambda=0.3$) & -- & $5\cdot10^{-4}$ & 92.91 & 98.50 & 96.85 & 69.83 & 35.43 & 81.20 & \textbf{79.12} \\
\algname{Supra-Mag} (BottomK, $\lambda=0.5$) & -- & $5\cdot10^{-4}$ & 91.39 & 96.83 & 93.90 & 62.40 & 17.72 & 79.30 & 73.59 \\
\algname{Supra-Mag} (BottomK, $\lambda=0.8$) & -- & $5\cdot10^{-4}$ & 87.34 & 99.33 & 87.60 & 64.06 & 29.53 & 77.80 & 74.28 \\
\bottomrule
\end{tabular}
}
\end{table}

\begin{table}[htbp]
\centering
\scriptsize
\caption{Test perplexity. Model: \texttt{Meta-Llama-3-8B}; fine-tuning dataset: Math17K; training budget: one epoch; rank-equivalent budget: \(r_0=8\). These are the same validation-selected checkpoints as \Cref{tab:llama8b_math17k_1epoch_accuracy_r8}. Full fine-tuning is an unbudgeted reference row separated by rules. Bold marks the lowest matched-budget adapter average PPL. The \textbf{Calib.} column gives the calibration source used to construct fixed \algname{Wanda}-style sparse supports; ``--'' denotes methods that do not use a calibration pass. Lower is better. Perplexities are computed on up to 120 examples per benchmark.}
\label{tab:llama8b_math17k_1epoch_ppl_r8}
\resizebox{\textwidth}{!}{%
\begin{tabular}{llcccccccc}
\toprule
\textbf{Method} & \textbf{Calib.} & \textbf{Selected LR} & \textbf{AddSub} & \textbf{MultiArith} & \textbf{SingleEq} & \textbf{GSM8K} & \textbf{AQuA} & \textbf{SVAMP} & \textbf{Average} \\
\midrule
\textsc{Base} (frozen) & -- & -- & 2.15 & 2.28 & 2.30 & 1.98 & 2.64 & 2.41 & 2.22 \\
\midrule
Full fine-tuning & -- & $5\cdot10^{-5}$ & 1.07 & 1.08 & 1.09 & 1.17 & 1.35 & 1.13 & 1.14 \\
\midrule
\algname{LoRA} & -- & $5\cdot10^{-4}$ & 1.07 & 1.09 & 1.10 & 1.17 & 1.33 & 1.13 & 1.14 \\
\algname{RoSA} & -- & $10^{-4}$ & 1.10 & 1.12 & 1.13 & 1.21 & 1.42 & 1.17 & 1.18 \\
\algname{SIFT} (TopK) & -- & $10^{-4}$ & 1.10 & 1.12 & 1.13 & 1.21 & 1.42 & 1.17 & 1.18 \\
\algname{SIFT} (RandK) & -- & $10^{-3}$ & 1.07 & 1.08 & 1.09 & 1.17 & 1.31 & 1.13 & 1.13 \\
\algname{Super} (RandK) & -- & $10^{-3}$ & 1.07 & 1.08 & 1.09 & 1.17 & 1.31 & 1.12 & \textbf{1.13} \\
\algname{Super} (TopK) & C4 & $5\cdot10^{-4}$ & 1.08 & 1.09 & 1.10 & 1.18 & 1.35 & 1.14 & 1.15 \\
\algname{Super} (BottomK) & C4 & $10^{-3}$ & 1.07 & 1.08 & 1.09 & 1.17 & 1.31 & 1.13 & 1.13 \\
\algname{Magnitude} (TopK) & -- & $10^{-3}$ & 1.07 & 1.08 & 1.09 & 1.17 & 1.32 & 1.13 & 1.13 \\
\algname{Magnitude} (BottomK) & -- & $10^{-3}$ & 1.06 & 1.08 & 1.09 & 1.16 & 1.30 & 1.12 & 1.13 \\
\algname{Supra} (TopK, $\lambda=0.3$) & C4 & $10^{-4}$ & 1.09 & 1.11 & 1.12 & 1.20 & 1.39 & 1.15 & 1.16 \\
\algname{Supra} (TopK, $\lambda=0.5$) & C4 & $10^{-4}$ & 1.09 & 1.11 & 1.12 & 1.20 & 1.40 & 1.16 & 1.17 \\
\algname{Supra} (TopK, $\lambda=0.8$) & C4 & $10^{-4}$ & 1.09 & 1.12 & 1.12 & 1.20 & 1.41 & 1.16 & 1.17 \\
\algname{Supra} (BottomK, $\lambda=0.3$) & C4 & $5\cdot10^{-4}$ & 1.07 & 1.08 & 1.10 & 1.17 & 1.32 & 1.13 & 1.14 \\
\algname{Supra} (BottomK, $\lambda=0.5$) & C4 & $5\cdot10^{-4}$ & 1.07 & 1.08 & 1.10 & 1.17 & 1.32 & 1.13 & 1.14 \\
\algname{Supra} (BottomK, $\lambda=0.8$) & C4 & $5\cdot10^{-4}$ & 1.07 & 1.08 & 1.10 & 1.17 & 1.32 & 1.13 & 1.14 \\
\algname{Supra-Mag} (BottomK, $\lambda=0.3$) & -- & $5\cdot10^{-4}$ & 1.07 & 1.08 & 1.10 & 1.17 & 1.32 & 1.13 & 1.14 \\
\algname{Supra-Mag} (BottomK, $\lambda=0.5$) & -- & $5\cdot10^{-4}$ & 1.07 & 1.08 & 1.10 & 1.17 & 1.31 & 1.12 & 1.13 \\
\algname{Supra-Mag} (BottomK, $\lambda=0.8$) & -- & $5\cdot10^{-4}$ & 1.07 & 1.08 & 1.10 & 1.17 & 1.32 & 1.13 & 1.14 \\
\bottomrule
\end{tabular}
}
\end{table}

\clearpage

\subsubsection{Math17K, three-epoch comparison}
\label{sec:llama8b_math17k_appendix}

\Cref{tab:llama8b_math17k_3epoch_accuracy_r8,tab:llama8b_math17k_3epoch_ppl_r8} report the \texttt{Meta-Llama-3-8B} Math17K comparison summarized in the main text. All rows in these tables use the three-epoch schedule. Full fine-tuning is included as an unbudgeted reference row. The corresponding one-epoch comparison is reported separately in \Cref{tab:llama8b_math17k_1epoch_accuracy_r8,tab:llama8b_math17k_1epoch_ppl_r8}, and the hybrid TopK--BottomK \algname{Super} variants are reported separately in \Cref{tab:llama8b_math17k_hybrid_super_accuracy_r8}.

\begin{table}[htbp]
\centering
\scriptsize
\caption{Arithmetic benchmark accuracy. Model: \texttt{Meta-Llama-3-8B}; fine-tuning dataset: Math17K; training budget: three epochs; rank-equivalent budget: \(r_0=8\). Adapter rows use approximately the rank-8 \algname{LoRA} trainable-parameter budget for this model. Full fine-tuning is an unbudgeted reference row separated by rules. Bold marks the best observed matched-budget adapter average. The \textbf{Calib.} column gives the calibration source used to construct fixed \algname{Wanda}-style sparse supports; ``--'' denotes methods that do not use a calibration pass. Each row uses the learning rate selected by the held-out validation split of the fine-tuning set. Results are exact-answer accuracy (\%).}
\label{tab:llama8b_math17k_3epoch_accuracy_r8}
\resizebox{\textwidth}{!}{%
\begin{tabular}{llcccccccc}
\toprule
\textbf{Method} & \textbf{Calib.} & \textbf{Selected LR} & \textbf{AddSub} & \textbf{MultiArith} & \textbf{SingleEq} & \textbf{GSM8K} & \textbf{AQuA} & \textbf{SVAMP} & \textbf{Average} \\
\midrule
\textsc{Base} (frozen) & -- & -- & 22.53 & 21.00 & 36.42 & 10.24 & 24.02 & 24.10 & 23.05 \\
\midrule
Full fine-tuning & -- & $5\cdot10^{-5}$ & 90.63 & 96.17 & 93.70 & 54.36 & 38.19 & 76.20 & 74.87 \\
\midrule
\algname{LoRA} & -- & $10^{-4}$ & 76.71 & 83.00 & 87.99 & 67.10 & 34.25 & 78.60 & 71.27 \\
\algname{RoSA} & -- & $5\cdot10^{-5}$ & 34.43 & 97.83 & 34.06 & 40.71 & 40.55 & 30.60 & 46.36 \\
\algname{SIFT} (TopK) & -- & $5\cdot10^{-5}$ & 45.82 & 95.50 & 82.87 & 68.99 & 42.52 & 78.80 & 69.08 \\
\algname{SIFT} (RandK) & -- & $10^{-4}$ & 26.08 & 93.33 & 66.34 & 61.18 & 44.09 & 49.30 & 56.72 \\
\algname{Super} (RandK) & -- & $10^{-4}$ & 0.25 & 97.33 & 48.43 & 59.97 & 37.80 & 31.00 & 45.80 \\
\algname{Super} (TopK) & C4 & $5\cdot10^{-5}$ & 4.81 & 88.17 & 13.19 & 9.48 & 35.04 & 5.90 & 26.10 \\
\algname{Super} (BottomK) & C4 & $10^{-4}$ & 72.15 & 95.67 & 88.58 & 66.34 & 42.91 & 70.70 & \textbf{72.73} \\
\algname{Magnitude} (TopK) & -- & $10^{-4}$ & 13.92 & 97.33 & 73.82 & 64.52 & 44.88 & 61.70 & 59.36 \\
\algname{Magnitude} (BottomK) & -- & $10^{-4}$ & 9.11 & 94.67 & 45.28 & 44.05 & 41.34 & 23.30 & 42.96 \\
\algname{Supra} (TopK, $\lambda=0.3$) & C4 & $5\cdot10^{-5}$ & 35.95 & 90.00 & 52.36 & 60.80 & 41.34 & 32.70 & 52.19 \\
\algname{Supra} (TopK, $\lambda=0.5$) & C4 & $5\cdot10^{-5}$ & 17.22 & 96.67 & 51.18 & 66.79 & 38.98 & 40.70 & 51.92 \\
\algname{Supra} (TopK, $\lambda=0.8$) & C4 & $5\cdot10^{-5}$ & 0.00 & 97.67 & 0.00 & 2.35 & 32.68 & 0.00 & 22.12 \\
\algname{Supra} (BottomK, $\lambda=0.3$) & C4 & $10^{-4}$ & 35.70 & 49.83 & 38.98 & 50.72 & 40.16 & 14.30 & 38.28 \\
\algname{Supra} (BottomK, $\lambda=0.5$) & C4 & $10^{-3}$ & 85.32 & 85.00 & 75.79 & 24.03 & 33.46 & 69.60 & 62.20 \\
\algname{Supra} (BottomK, $\lambda=0.8$) & C4 & $10^{-4}$ & 60.76 & 94.17 & 66.34 & 67.40 & 38.98 & 54.50 & 63.69 \\
\algname{Supra-Mag} (BottomK, $\lambda=0.3$) & -- & $5\cdot10^{-5}$ & 6.08 & 72.17 & 21.06 & 49.58 & 37.01 & 7.90 & 32.30 \\
\algname{Supra-Mag} (BottomK, $\lambda=0.5$) & -- & $10^{-4}$ & 60.51 & 50.67 & 54.53 & 48.82 & 36.22 & 20.60 & 45.22 \\
\algname{Supra-Mag} (BottomK, $\lambda=0.8$) & -- & $10^{-4}$ & 64.30 & 97.50 & 68.11 & 67.32 & 36.22 & 62.10 & 65.93 \\
\bottomrule
\end{tabular}
}
\end{table}

\begin{table}[htbp]
\centering
\scriptsize
\caption{Test perplexity. Model: \texttt{Meta-Llama-3-8B}; fine-tuning dataset: Math17K; training budget: three epochs; rank-equivalent budget: \(r_0=8\). These are the same validation-selected checkpoints as \Cref{tab:llama8b_math17k_3epoch_accuracy_r8}. Full fine-tuning is an unbudgeted reference row separated by rules. Bold marks the lowest matched-budget adapter average PPL. The \textbf{Calib.} column gives the calibration source used to construct fixed \algname{Wanda}-style sparse supports; ``--'' denotes methods that do not use a calibration pass. Lower is better. Perplexities are computed on up to 120 examples per benchmark.}
\label{tab:llama8b_math17k_3epoch_ppl_r8}
\resizebox{\textwidth}{!}{%
\begin{tabular}{llcccccccc}
\toprule
\textbf{Method} & \textbf{Calib.} & \textbf{Selected LR} & \textbf{AddSub} & \textbf{MultiArith} & \textbf{SingleEq} & \textbf{GSM8K} & \textbf{AQuA} & \textbf{SVAMP} & \textbf{Average} \\
\midrule
\textsc{Base} (frozen) & -- & -- & 2.15 & 2.28 & 2.30 & 1.98 & 2.64 & 2.41 & 2.22 \\
\midrule
Full fine-tuning & -- & $5\cdot10^{-5}$ & 1.03 & 1.04 & 1.04 & 1.09 & 1.17 & 1.06 & 1.07 \\
\midrule
\algname{LoRA} & -- & $10^{-4}$ & 1.06 & 1.07 & 1.08 & 1.15 & 1.34 & 1.11 & 1.12 \\
\algname{RoSA} & -- & $5\cdot10^{-5}$ & 1.08 & 1.10 & 1.11 & 1.19 & 1.37 & 1.15 & 1.16 \\
\algname{SIFT} (TopK) & -- & $5\cdot10^{-5}$ & 1.08 & 1.10 & 1.10 & 1.19 & 1.34 & 1.14 & 1.15 \\
\algname{SIFT} (RandK) & -- & $10^{-4}$ & 1.10 & 1.13 & 1.15 & 1.23 & 1.44 & 1.19 & 1.19 \\
\algname{Super} (RandK) & -- & $10^{-4}$ & 1.29 & 1.16 & 1.35 & 1.35 & 1.43 & 1.40 & 1.33 \\
\algname{Super} (TopK) & C4 & $5\cdot10^{-5}$ & 1.09 & 1.12 & 1.18 & 1.23 & 1.36 & 1.21 & 1.19 \\
\algname{Super} (BottomK) & C4 & $10^{-4}$ & 1.20 & 1.20 & 1.29 & 1.33 & 1.48 & 1.36 & 1.31 \\
\algname{Magnitude} (TopK) & -- & $10^{-4}$ & 1.16 & 1.13 & 1.20 & 1.26 & 1.52 & 1.24 & 1.23 \\
\algname{Magnitude} (BottomK) & -- & $10^{-4}$ & 1.09 & 1.11 & 1.12 & 1.20 & 1.36 & 1.16 & 1.17 \\
\algname{Supra} (TopK, $\lambda=0.3$) & C4 & $5\cdot10^{-5}$ & 1.07 & 1.08 & 1.09 & 1.17 & 1.32 & 1.12 & 1.13 \\
\algname{Supra} (TopK, $\lambda=0.5$) & C4 & $5\cdot10^{-5}$ & 1.07 & 1.08 & 1.09 & 1.17 & 1.33 & 1.12 & 1.13 \\
\algname{Supra} (TopK, $\lambda=0.8$) & C4 & $5\cdot10^{-5}$ & 1.13 & 1.14 & 1.12 & 1.18 & 1.34 & 1.17 & 1.17 \\
\algname{Supra} (BottomK, $\lambda=0.3$) & C4 & $10^{-4}$ & 1.06 & 1.06 & 1.08 & 1.15 & 1.27 & 1.11 & 1.12 \\
\algname{Supra} (BottomK, $\lambda=0.5$) & C4 & $10^{-3}$ & 1.04 & 1.05 & 1.05 & 1.11 & 1.19 & 1.08 & \textbf{1.08} \\
\algname{Supra} (BottomK, $\lambda=0.8$) & C4 & $10^{-4}$ & 1.06 & 1.06 & 1.08 & 1.15 & 1.29 & 1.11 & 1.12 \\
\algname{Supra-Mag} (BottomK, $\lambda=0.3$) & -- & $5\cdot10^{-5}$ & 1.07 & 1.08 & 1.10 & 1.18 & 1.35 & 1.13 & 1.14 \\
\algname{Supra-Mag} (BottomK, $\lambda=0.5$) & -- & $10^{-4}$ & 1.06 & 1.06 & 1.08 & 1.15 & 1.28 & 1.11 & 1.12 \\
\algname{Supra-Mag} (BottomK, $\lambda=0.8$) & -- & $10^{-4}$ & 1.06 & 1.06 & 1.08 & 1.15 & 1.27 & 1.11 & 1.11 \\
\bottomrule
\end{tabular}
}
\end{table}

\end{document}